\documentclass[12pt]{article}
\usepackage[letterpaper,left=1truein,right=1truein,top=1truein,bottom=1truein]{geometry}
\usepackage[utf8]{inputenc}
\usepackage[OT1]{fontenc}
\usepackage[dvipsnames]{xcolor}
\usepackage{graphicx}
\usepackage{booktabs}
\usepackage{longtable}
\usepackage{wrapfig}
\usepackage{rotating}
\usepackage[normalem]{ulem}
\usepackage{amsmath}
\usepackage{amssymb}
\usepackage{amsfonts}
\usepackage{mathrsfs}
\usepackage{bbm}
\usepackage{bm}
\usepackage{amsthm}
\usepackage{enumitem}
\usepackage{capt-of}
\usepackage{hyperref}
\usepackage{float}
\usepackage{subcaption}
\usepackage[section]{placeins}
\usepackage{setspace}

\usepackage[linesnumbered, ruled, vlined]{algorithm2e}

\newenvironment{smashedalign*}
{\par$\!\aligned}
{\endaligned$\par}

\theoremstyle{definition}
\newtheorem{corollary}{Corollary}%
\newtheorem{example}{Example}%
\newtheorem{lemma}{Lemma}%
\newtheorem{theorem}{Theorem}%
\newtheorem{definition}{Definition}%
\newtheorem{remark}{Remark}%
\newtheorem{assumption}{Assumption}

\newcommand{\wt}[1]{\widetilde{#1}}
\newcommand{\wh}[1]{\widehat{#1}}

\def\pr{\mbox{Pr}} %
\def\rank{\mathrm{rank}}

\DeclareMathOperator*{\argmin}{arg\,min}

\DeclareMathOperator*{\argmax}{arg\,max}

\newcommand{\EE}{{\mathbb{E}}}

\newcommand{\II}{{\mathbb{I}}}

\newcommand{\NN}{{\mathbb{N}}}

\newcommand{\PP}{{\mathbb{P}}}

\newcommand{\RR}{{\mathbb{R}}}

\newcommand{\bI}{{\mathbf{I}}}

\newcommand{\bu}{{\mathbf{u}}}
\newcommand{\bv}{{\mathbf{v}}}
\newcommand{\bw}{{\mathbf{w}}}
\newcommand{\boldeta}{\boldsymbol{\eta}}

\newcommand{\bmu}{{\boldsymbol{\mu}}}
\newcommand{\btheta}{{\boldsymbol{\theta}}}
\newcommand{\bpsi}{{\boldsymbol{\psi}}}
\newcommand{\bdelta}{{\boldsymbol{\delta}}}
\newcommand{\bbeta}{{\boldsymbol{\beta}}}

\newcommand{\mA}{{\mathcal{A}}}
\newcommand{\mB}{{\mathcal{B}}}
\newcommand{\mC}{{\mathcal{C}}}
\newcommand{\mD}{{\mathcal{D}}}

\newcommand{\mF}{{\mathcal{F}}}
\newcommand{\mG}{{\mathcal{G}}}
\newcommand{\mH}{{\mathcal{H}}}

\newcommand{\mM}{{\mathcal{M}}}
\newcommand{\mN}{{\mathcal{N}}}

\newcommand{\mQ}{{\mathcal{Q}}}

\newcommand{\mS}{{\mathcal{S}}}

\newcommand{\mU}{{\mathcal{U}}}
\newcommand{\mV}{{\mathcal{V}}}

\newcommand{\mX}{{\mathcal{X}}}
\newcommand{\mY}{{\mathcal{Y}}}

\newcommand{\fC}{{\mathfrak{C}}}

\newcommand{\norm}[1]{\| #1 \|}

\def\pen{{\mathscr{P}}}

\newcommand{\bigO}{\ensuremath{\mathop{}\mathopen{}\mathcal{O}\mathopen{}}}

\newcommand{\bigOp}{\bigO_\mathrm{p}}

\newcommand{\pushright}[1]{\ifmeasuring@#1\else\omit\hfill$\displaystyle#1$\fi\ignorespaces}

\def\argmax{\operatorname{argmax}}

\renewcommand{\liminf}{\varliminf}
\renewcommand{\limsup}{\varlimsup}

\def\given{\mid}

\def\ds1{{\mathrm{1 \hspace{-2.6pt} I}}}

\newcommand{\blue}[1]{#1}

\hypersetup{
 colorlinks=true,
 linkcolor=red,
 urlcolor=blue, 
 citecolor=blue}
\usepackage{natbib}

\title{Reinforcement Learning for Individual Optimal Policy from Heterogeneous Data\thanks{The content is solely the responsibility of the authors and does not necessarily represent the official views of the National Institutes of Health.}}
\author{Rui Miao\thanks{National Heart, Lung, and Blood Institute}, Babak Shahbaba\thanks{University of California, Irvine}, Annie Qu\thanks{University of California, Irvine}}
\date{}

\begin{document}

\def\spacingset#1{\renewcommand{\baselinestretch}%
{#1}\small\normalsize} \spacingset{1}

\maketitle

\begin{abstract}
Offline reinforcement learning (RL) aims to find optimal policies in dynamic
environments in order to maximize the expected total rewards by leveraging
pre-collected data. Learning from heterogeneous data is one of the
fundamental challenges in offline RL. Traditional methods focus on learning an
optimal policy for all individuals with pre-collected data from a single
episode or homogeneous batch episodes, and thus, may result in a suboptimal
policy for a heterogeneous population.
In this paper, we propose an individualized offline policy optimization framework
for heterogeneous time-stationary Markov decision processes (MDPs). The proposed
heterogeneous model with individual latent variables enables us to
efficiently estimate the individual Q-functions, and our Penalized Pessimistic
Personalized Policy Learning (P4L)
framework admits suitable local solutions with a fast rate on the average regret under a
within-group partial coverage assumption and regularity conditions. In addition, our simulation studies and a real data application demonstrate the superior numerical
performance of the proposed method compared with existing methods.
\end{abstract}

\noindent%
\vfill

\newpage
\spacingset{1.6}

\allowdisplaybreaks

\section{Introduction}
\label{sec: introduction}
Reinforcement learning (RL) has emerged as a powerful tool for numerous
decision-making problems in variety of domains, including healthcare, robotics,
gaming and pricing strategy. In practice, however, different individuals can exhibit substantial variations in their behaviors and responses
to different actions, together with other individual features, leading to high
population heterogeneity. Ignoring the resulting heterogeneity can lead to suboptimal
policies \citep{chen2022reinforcement,hu2022doubly}, especially for some
individuals who are underrepresented, disadvantaged or vulnerable in the
population. For example, this could lead to health disparities for
populations living in rural areas.
The commonly used RL algorithms \citep[e.g.,][]{munos2003error,mnih2015human,haarnoja2018soft,luckett2019estimating,chen2021decision} assume the environment is stationary and homogeneous for
all individuals, which may not hold in general. Therefore, there is a growing
need to develop RL methods which can learn individualized policies to account for
populational heterogeneity.

Recently, several RL methods for heterogeneous data have been developed to
address the suboptimality of learning policy with a homogeneity assumption. Specifically,
PerSim \citep{agarwal2021persim} learns personalized simulators by assuming a
low rank tensor structure of latent factors for each individual, state, and
action. However, it requires a finite state space to find an optimal policy based on
the learned personalized environment. \citet{chen2022reinforcement} introduce
heterogeneity only on the reward function, where the Q-functions and groups are
learned simultaneously, and the policies are updated groupwise. However, individual
policies are learned within a group, and information borrowing among groups is
ignored. \citet{hu2022doubly} learn the warm-up groupwise policies from the
last period of stationary episodes based on change-point detection. However,
the methods of \citet{chen2022reinforcement} and \citet{hu2022doubly} can only make use of
information within the learned population groups, hence sample efficiency is not
fully utilized.
Meta-reinforcement learning (meta-RL) aims to learn a policy that can maximize
the reward of any task from a distribution of heterogeneous tasks
\citep{beck2023survey,pmlr-v139-mitchell21a,zhang2021provably}. However, existing meta-RL methods typically
require the collection of millions of interactions with environments with
different tasks during meta-training online or more pre-collected training data.
In addition, theoretical understanding of offline meta-RL remains limited.
For a single decision point, evaluating the heterogeneous treatment effect (HTE) is
inevitable for learning personalized treatment regimes. Existing works mainly
estimate HTE with panel or non-panel data
\citep[][and references
therein]{athey2015machine,shalit2017estimating,wager2018estimation,kunzel2019metalearners,nie2021quasi,shen2022heterogeneous}.
Our goal is to develop a single task RL method with offline data from individuals having
heterogeneous treatment effects and transactions.

The problem we intend to solve
has an element in between the meta-RL and HTE.
More specifically, we aim to solve an offline reinforcement learning (RL) problem under
populational heterogeneity, where we learn subject-specific policies based on offline
datasets collected a priori. One challenge of the populational heterogeneity in offline RL
is that each subject may have different state-action transitions with different
immediate rewards. As a result, existing batch RL methods assuming subject
homogeneity, aiming to learn a common optimal policy for all subjects, may
suffer from biased and inconsistent estimation of subject-specific value
functions and consequently, only sub-optimal policies can be learned.
However, directly applying batch RL methods to detected homogeneous clusters of trajectories
\citep{chen2022reinforcement,hu2022doubly} diminishes sample
efficiency, as cross-subject information is not incorporated in policy learning.
More importantly, the coverage assumption becomes less feasible for a specific homogeneous
subgroup rather than for the entire population.
In addition, since the offline data follow some fixed behavior policies, and we
cannot interact with the environments of the subjects to collect more data
to evaluate the policies, this may result in a distributional shift between the
behavior policies and the optimal ones \citep{levine2020offline}. The
distributional shift increases the estimation error in evaluating the value of behavior
policies, and hence causes suboptimality in learning optimal individual policies.

To tackle the above challenges, in this paper we propose a novel
heterogeneous latent variable model with a penalized pessimistic policy learning approach
to learn optimal individual policies simultaneously.
Our method relies on a partial coverage assumption that the
average visitation probability of the batch data from individuals within each
homogeneous subgroup can cover the discounted visitation probability induced by the target policy for each individual in that subgroup.
While directly applying the RL method for an individual episode requires that the
visitations of state-action from an individual episode should cover
the visitation probability induced by the target policy. This is an unrealistic
assumption and could prevent us from
borrowing information from other individuals in the population. By introducing a
model of the shared structure of the Q-function and policy with individual latent variables,
which captures heterogeneous individual information, we could efficiently
utilize the aggregated data from all individual episodes for off-policy evaluation.

In addition, the proposed heterogeneous latent variable model allows us to simultaneously
learn optimal policies for all individuals, instead of learning optimal policy
for each estimated homogeneous cluster, respectively \citep{chen2022reinforcement,hu2022doubly}.
To further weaken the coverage assumption for offline RL, we adopt the pessimism
concept and establish a policy optimization algorithm to find the optimal
policy structure with individual latent variables, and aim to
maximize the overall value estimated by the most pessimistic policy evaluator
from an uncertainty set of Q function candidates. With a proper uncertainty
level, we only require a partial coverage assumption for the optimal policies.

By introducing multi-centroid penalties on latent variables, we could encourage
individuals from the same subgroup to have similar latent variables and hence
similar optimal policies. Theoretically, under the stated regularity and tuning
conditions, we show that there exist local penalized estimators
of optimal policies with unknown subgroup information that attain the same regret rate
as oracle estimators when the subgroup information is known. Their regrets are
upper bounded by a rate near the inverse square root of the number of transitions from
heterogeneous data. %
To resolve the
computational burden introduced by the constraints of the uncertainty set, we propose
to solve a Lagrangian dual problem which also admits local solutions with the same rate of the
upper bound of regret, under an additional assumption on the convexity of
the space of the Q-function.

In practice, our method can be applied to learning optimal individual policies
for the same task, but with heterogeneous time-stationary environments. We
remark that the time-stationarity and Markovianity can be ensured by
concatenating multiple decision points and adding auxiliary variables (e.g.,
time stamps) into the state. Therefore, the proposed method is also applicable
for tasks with rapidly changing environments, for example, critical care. In
addition, the theoretical properties established in our penalized estimation method
guarantee that our method is applicable when the number of
individuals is small but the length of each trajectory is large, which is
common in mobile health applications.

The rest of the paper is organized as follows. In Section \ref{sec:
  preliminaries}, we briefly introduce the discrete-time time-stationary
heterogeneous MDP and the problem formulations. Then we formally introduce the
proposed method for policy optimization with heterogeneous batch data in Section
\ref{sec: methods}. Theoretical analysis on the regret of policy learned by
the proposed algorithm is given in Section \ref{sec: theory}. In Section 5, we
further illustrate the implementation details. Sections 6 and 7 demonstrate the
numerical performance of our method with simulation studies and a real data application.
Section 8 concludes this paper.

\section{Preliminaries}
\label{sec: preliminaries}
In this section, we provide notations and frameworks for the Markov decision
processes. In addition, we provide several important concepts and
properties associated with MDPs, which will be used as building blocks for the proposed method.

\subsection{Notations}
We denote the index set $[N] = \left\{ 1,\dots,N \right\}$. For an index set
$\mG\subseteq \NN$, we denote its cardinality by $|\mG|$. For
sequences $\left\{ \varpi_n \right\}$ and $\left\{ \theta_n \right\}$, we use
$\varpi_n\gtrsim \theta_n$ to denote that $\varpi_n \geq c \theta_n$ for some constant $c>0$, and
$\varpi_n\lesssim \theta_n$ to denote that $\varpi_n\leq c' \theta_n$ for some constant $c'>0$.
For two measures $\nu$ and $\mu$, we use $\nu<\!\!<\mu$ to denote that $\nu$ is absolutely
continuous with respect to $\mu$.
For a collection $\bu = \{u^i\}_{i\in[N]}$, we define $\norm{\bu}_{2,\infty} = \max_{i\in[N]}\norm{u^i}_2$.
In addition, we use $(S^i,A^i,R^i,S_+^i)$ or $(S^i,A^i,S_+^i)$ to denote a
transition tuple (state, action, reward, next state) or (state, action, next state) in the trajectory of subject $i$ for any decision time $t$.
\subsection{Frameworks}
Reinforcement learning is a technique that solves sequential decision making
problems for agents in some unknown environments. The observed data can be
summarized in sequences of state-action-reward tuples over decision time. At
each decision time $t \geq 0$, the agent $i\in[N]$ observes the current
state $S_t^i\in\mS$ from the environment, where the state space $\mS\subseteq \RR^{d_s}$ is a Borel set.
Based on the observed history, a decision $A_t^i$ is selected from the action
space $\mA$, which is assumed to be a finite discrete set of size $|\mA|$ in this paper. The
environment then provides an immediate reward $R_t^i\in\RR$ to the agent
and moves forward to the next state $S_{t+1}^i$.

The policy $\pi^i = \left\{ \pi_t^i \right\}_{t \geq 0}$ for agent $i \in [N]$ is a set
of probability distributions on the action space $\mA$, which determines the
behavior of the agent.
In this work, we focus on the time-invariant policy, i.e., $\pi_0^i= \dots = \pi_t^i = \dots$, for each
$i\in[N]$, respectively.
For each policy $\pi^i$, we define the Q-function to measure the expected
discounted cumulative reward for the agent $i$ starting from any state-action
pair $(s,a)\in \mS\times\mA$, denoted by
\begin{equation}
  \label{eq: Q-function}
  Q_{\pi^i}^i(s,a) =\EE^{\pi^i}\left\{ \sum_{t=0}^{\infty}\gamma^t R_t^i \given S_0^i=s, A_0^i=a \right\},
\end{equation}
where the expectation $\EE^{\pi^i}$ is taken by assuming that the agent $i$
selects actions according to the policy $\pi^i$, and $\gamma\in[0,1)$ is the discounted
factor for balancing the short-term and long-term rewards. For example, a myopic
value function (with $\gamma=0$) only cares about immediate reward, which is
typically adopted in contextual bandit problems where the current state is
independent of history.

To evaluate the overall performance of a policy $\pi^i$ for individual $i$, we
define the value of $\pi^i$ by
\begin{equation}
  \label{eq: value}
J^i(\pi^i) = (1-\gamma) \EE_{S_0^i\sim\nu^i} Q_{\pi^i}^i(S_0^i, \pi^i(S_0^i)),
\end{equation}
for some given distribution $\nu^i$ of the initial state $S_0^i$. Here we define
$Q(s,\pi(s)) = \sum_{a\in\mA} Q(s,a)\\\times \pi(a\given s)$ for any policy $\pi$ and function $Q$ defined over $\mS\times\mA$.
The objective of individualized policy optimization is to find an in-class policy
\begin{equation}
\label{eq: optimal policy}
\pi_{*}^{i}\in\arg\max_{\pi^i\in\Pi} J^i(\pi^i),
\end{equation}
which optimizes the value for each individual
$i\in[N]$. In this work, we focus on optimizing the values of policies for a shared distribution
$\nu$ of the initial states for every individual, i.e., $\nu_i = \nu$ for all $i\in [N]$.
Making decisions tailored for each individual is critical in many applications.
For example, in mobile
health, we only consider applying some treatment regimes once the state of an individual
reaches a certain range of risk.

To formalize the framework, we adopt heterogeneous time-stationary Markov decision processes
\citep[MDPs,][]{puterman2014markov} to model the data generating processes for
different agents. Specifically, for
each individual $i\in [N]$, the episode
$\left\{ \left( S_t^i,A_t^i,R_t^i \right) \right\}_{t \geq 0}$ follows an MDP
$\mM^{(i)} = \left\{ \mS, \mA, \PP^{i}, r^{i}, \gamma \right\}$,
which satisfies the standard conditions in Assumption \ref{ass: standard} below.
\begin{assumption}[Standard Conditions]
  \label{ass: standard}
The following statements hold for MDPs $\left\{ \mM^{(i)} \right\}_{i\in[N]}$.
\begin{enumerate}[label=(\alph*)]
  \item There exists a time-invariant transition kernel $\PP^i$ such that for any
        $t\geq 0$, $s\in\mS,a\in\mA$ and $F\in\mB(\mS)$,
\begin{equation*}
\pr \left(S_{t+1}^i\in F \given S_t^i=s, A_t^i=a, \left\{ S_j^i,A_j^i,R_j^i \right\}_{0\leq j <t}\right) = \PP^i(S_{t+1}^i\in F \given S_t^i=s, A_t^i=a),
\end{equation*}
        where $\mB(\mS)$ is the collection of Borel subsets of $\mS$ and
        $\left\{ S_j^i,A_j^i,R_j^i \right\}_{0\leq j <t}=\emptyset$ if $t=0$.
  \item There exists a reward function such that
        $\EE[R_t^i\given S_t^i,A_t^i, \left\{ S_{t'}^i,A_{t'}^i,R_{t'}^i \right\}_{0\leq t' <t}] = r^{i}(S_t^i,A_t^i)$
        for any $t\geq 0$.
        Assume that $|R_t^i|\leq R_{\max}$ almost surely for all $i,t$, and normalize rewards so that $R_{\max}=1$. In particular, $\|r^i\|_\infty\leq1$.
 \item The batch data $\left\{ (S_t^i,A_t^i,R_t^i,S_{t+1}^i) \right\}_{0\leq t <T^i}$ for subject $i$ are generated by a stationary behavior policy $\pi_b^i: \mS\rightarrow\Delta(\mA)$, which is a mapping from state space $\mS$ to the probability simplex on action space $\mA$.
To simplify notations, we assume a balanced dataset with $T_i=T$ here. In case
        of unbalanced data, the proposed method can be extended without much difficulty.
\end{enumerate}
\end{assumption}
Assumption \ref{ass: standard} is standard in the RL literature for
homogeneous MDP, whereas in our heterogeneous setting, the stationary transition kernels
$\PP^i$ are allowed to be different among individuals. The uniformly bounded assumption on the immediate reward $R_t^i$ is used to
simplify the technical proofs but can be relaxed by imposing some higher-order
moment condition on $R_t^i$ instead.
We optimize stationary policies in a pre-specified policy class $\Pi$.
Existence of an in-class maximizer also requires regularity of that class.
The compactness and continuity assumptions in Section \ref{sec: assumptions}
provide it for the theoretical analysis. Some commonly used policy classes include linear decision
functions, decision trees and neural networks.

We define the MDP induced probability measures, followed
by an important property in Lemma \ref{lem: OPE error}, which is a building
block for our method. For each MDP $\mM^{(i)}$, we further introduce the
discounted visitation probability measure over $\mS\times\mA$ induced by a policy
$\pi^i$ as follows:
\begin{equation}
\label{eq: discounted visitation probability}
d_{\pi^i}^i = (1-\gamma) \sum_{t=0}^{\infty} \gamma^t p_{\pi^i,t}^i,
\end{equation}
where $p_{\pi^i,t}^i$ is the marginal probability measure of $(S_t^i,A_t^i)$
induced by the policy $\pi^i$ with the initial state distribution $\nu$.
Similarly, we define the average visitation probability across $T$ decision
points as
\begin{equation}
\label{eq: average visitation probability}
\bar{d}^i = \frac{1}{T}\sum_{t=0}^{T-1} p_{\pi_b^i,t}^i,
\end{equation}
Here $p_{\pi_b^i,t}^i$ denotes the marginal law in the observed behavior
trajectory, initialized from its own data-collection distribution. In the
stationary analysis below this is an invariant state distribution for
$\pi_b^i$ and $\PP^i$, and it need not equal the evaluation distribution
$\nu$ used in \eqref{eq: discounted visitation probability}. The corresponding expectation is denoted by $\bar{\EE}^i$.
The grand average visitation probability over all MDPs across $T$ decision points
is defined by $\bar{d} = N^{-1} \sum_{i\in[N]} \bar{d}^i$ with corresponding
expectation $\bar{\EE}$.

With the above definitions of MDP induced probability measures, we can quantify the
off-policy evaluation (OPE) for estimating the value of $\pi^i$ for
individual $i$ using Q-function $\wt{Q}^i$ in \eqref{eq: value}, which is
denoted as $\wt{J}^i(\pi^i) = (1-\gamma) \EE_{S_0^i\sim \nu} \wt{Q}^i(S_0^i, \pi^i(S_0^i))$.
Then we have the following property for the OPE error of $\wt{J}^i(\pi^i)$ using $\wt{Q}^i$.

\begin{lemma}
  \label{lem: OPE error}
  Under Assumption \ref{ass: standard}, for any bounded measurable candidate
  $\wt{Q}^i$, we have that
\begin{equation}
  \label{eq: OPE error}
  J^i(\pi^i) - \wt{J}^i (\pi^i) = \EE_{(S^i,A^i)\sim d_{\pi^i}^i} [R^i + \gamma \wt{Q}^i(S_+^i,\pi^i(S_+^i)) - \wt{Q}^i(S^i,A^i)].
\end{equation}
Here the reward and next-state expectations, conditional on $(S^i,A^i)$,
are taken under individual $i$'s reward law and transition kernel.
\end{lemma}
Lemma \ref{lem: OPE error} implies that the OPE error by $\wt{Q}^i$ of policy
$\pi^i$ is the expectation of Bellman error with respect to the discounted visitation
probability $d_{\pi^i}^i$ induced by policy $\pi^i$ for individual $i$.
This identity controls policy-value error. A single scalar value error does
not, by itself, identify the entire Q-function. In a heterogeneous population,
we use empirical Bellman moments to obtain an upper bound on value error under
the shared model introduced in Section \ref{sec: methods}.

\section{Methods}
\label{sec: methods}
In this section, we present our value-based method for learning optimal individualized
policies with a heterogeneous latent variable model. Objective \eqref{eq:
  optimal policy} implies a straightforward approach to learn the optimal policy
$\pi_{*}^i$ by maximizing the value \eqref{eq: value} with an estimated Q-function
$Q_{\pi^i}^i$ defined in \eqref{eq: Q-function} for each individual $i\in[N]$ using existing RL methods for homogeneous data
\citep[e.g.,][]{munos2003error,mnih2015human,haarnoja2018soft,luckett2019estimating,chen2021decision}.
However, utilizing
individual-specific data only in estimating the heterogeneous Q-functions
$\left\{ Q_{\pi^i}^i \right\}_{i\in[N]}$ and subsequent learning optimal policies
$\left\{ \pi_{*}^i \right\}_{i\in[N]}$ could lead to sample-inefficient estimations,
especially when some individuals have a limited number of observations,
e.g., when the episode length $T$ is small.

\subsection{Heterogeneous Latent Variable Model}
To obtain more efficient estimation, it is beneficial to aggregate information
among individuals by imposing shared
structures on policies and Q-functions with latent variables that encode
individual information. Then we can utilize sub-homogeneous information by
encouraging the grouping of individuals with similar latent variables. This ensures
that individuals from the same subgroups have similar optimal policies, by construction.

Our proposed shared structures and learning algorithm is motivated by the
individual OPE errors.
If $d_{\pi^i}^i <\!\!\!< \bar{d}^i$ and $d_{\pi^i}^i/ \bar{d}^i \in \mF$, which is a
symmetric bounded class of functions on $\mS\times\mA$, Lemma \ref{lem: OPE error}
implies that
\begin{equation}
\label{eq: OPE bound}
\begin{aligned}
\left| J^i(\pi^i) - \wt{J}^i (\pi^i)  \right| &= \left| \bar{\EE}^i \left\{ \frac{d_{\pi^i}^i(S^i,A^i)}{\bar{d^i}(S^i,A^i)} \left( R^i + \gamma \wt{Q}^i(S_+^i,\pi^i(S_+^i)) - \wt{Q}^i(S^i,A^i)\right)\right\} \right|\\
  & \leq \sup_{f^i\in\mF} \bar{\EE}^i \left\{ f^i(S^i,A^i) \left( R^i + \gamma \wt{Q}^i(S_+^i,\pi^i(S_+^i)) - \wt{Q}^i(S^i,A^i)\right)\right\}.
\end{aligned}
\end{equation}
The upper bound in \eqref{eq: OPE bound} circumvents using the policy-induced discounted visitation
probability $d_{\pi^i}^i$ in \eqref{eq: OPE error}, which is typically hard
to estimate. As long as we can find
the space $\mF$, which contains $d_{\pi^i}^i/\bar{d}^i$, we can always establish an upper
bound of the OPE error by calculating the supremum of expectation in \eqref{eq: OPE
  bound} with respect to the empirical distribution $\bar{d}^i$ \eqref{eq: average visitation probability}.
Then a min-max estimating approach \citep[e.g.,][]{jiang2020minimax} to learn $Q_{\pi^i}^i$ can be formulated as
\begin{equation*}
\min_{\wt{Q^i}}\sup_{f^i\in\mF} \bar{\EE}^i \left\{ f^i(S^i,A^i) \left( R^i + \gamma \wt{Q}^i(S_+^i,\pi^i(S_+^i)) - \wt{Q}^i(S^i,A^i)\right)\right\}.
\end{equation*}

However, learning $Q_{\pi^i}^i$ separately for each individual suffers from sample
inefficiency and requires individual coverage. This naturally motivates us to
learn Q-functions simultaneously with a combined dataset for the heterogeneous
population. Suppose individuals belong to subgroups $\mG^1,\ldots,\mG^K$,
with $\PP^i=\PP^k$ and $r^i=r^k$ for $i\in\mG^k$. Let $p_k=|\mG^k|/N$ and
\begin{equation*}
\bar d^k=\frac{1}{|\mG^k|}\sum_{i\in\mG^k}\bar d^i.
\end{equation*}
We write $\bar{\EE}^k$ for expectation with $(S,A)\sim\bar d^k$, followed
by a transition from $\PP^k$ and a reward with conditional mean $r^k$.
This allows state--action pairs induced by a target policy to be covered by
behavior policies from other individuals in the same subgroup.

To avoid learning Q-functions individually, we propose shared structures
with latent variables $u^i=u^{\mG^k}\in\mU$ for $i\in\mG^k$, where
$\mU\subseteq\RR^{d_u}$ is compact. For policies
$\pi^k=\pi(\bullet;u^{\mG^k})$, suppose there exists $Q_\pi\in\mQ$ such that
$Q_\pi(\bullet;u^{\mG^k})=Q_{\pi^k}^k(\bullet)$. For a candidate
$\wt Q\in\mQ$, write $\wt Q^i=\wt Q(\bullet;u^{\mG^k})$ for $i\in\mG^k$.
Let $\mF$ be a bounded symmetric class of functions on
$\mS\times\mA\times\mU$. Assume $d_{\pi^k}^k\ll\bar d^k$ and that one
shared function $w_\pi\in\mF$ satisfies
\begin{equation*}
w_\pi(s,a;u^{\mG^k})=\frac{d_{\pi^k}^k}{\bar d^k}(s,a),
\qquad k\in[K],\quad\bar d^k\text{-almost everywhere}.
\end{equation*}

For brevity write the sampled Bellman residual as
\[
 b_{\wt Q,\pi}^k(S,A,R,S_+)
 =R+\gamma\wt Q(S_+,\pi(S_+;u^{\mG^k});u^{\mG^k})
       -\wt Q(S,A;u^{\mG^k}).
\]
Lemma \ref{lem: OPE error} and symmetry of $\mF$ give
\begin{equation}
\label{eq: total OPE bound}
\begin{aligned}
\left|\frac1N\sum_{i=1}^N\{J^i(\pi^i)-\wt J^i(\pi^i)\}\right|
 &=\left|\sum_kp_k\bar\EE^k
       [w_\pi(S,A;u^{\mG^k})b_{\wt Q,\pi}^k]\right|\\
 &\leq\sup_{f\in\mF}\sum_kp_k\bar\EE^k
       [f(S,A;u^{\mG^k})b_{\wt Q,\pi}^k].
\end{aligned}
\end{equation}
The bound controls the OPE error of the population-average value. Given the
latent variables, we can estimate $Q_\pi$ by solving the min-max problem
\begin{equation}
  \label{eq: OPE shared structure}
  \min_{\wt Q\in\mQ}\sup_{f\in\mF}
  \sum_{k=1}^Kp_k\bar{\EE}^k\left[
  f(S,A;u^{\mG^k})
  \{R+\gamma\wt Q(S_+,\pi(S_+;u^{\mG^k});u^{\mG^k})
        -\wt Q(S,A;u^{\mG^k})\}\right].
\end{equation}
Specifically, the empirical version of \eqref{eq: OPE shared structure} is
\begin{equation}
  \label{eq: OPE sample version}
  \min_{\wt Q\in\mQ}\sup_{f\in\mF}\wh\Phi(\wt Q,f,\pi,\bu),
\end{equation}
where
\begin{equation*}
\begin{aligned}
\wh\Phi(\wt Q,f,\pi,\bu)
=\frac{1}{NT}\sum_{i=1}^N\sum_{t=0}^{T-1}
 f(S_t^i,A_t^i;u^i)
 \{R_t^i+\gamma\wt Q(S_{t+1}^i,\pi(S_{t+1}^i;u^i);u^i)
             -\wt Q(S_t^i,A_t^i;u^i)\}.
\end{aligned}
\end{equation*}
When the groups are unknown, the empirical criterion is optimized jointly
with the latent variables.

To better understand the concept of shared structure, we illustrate the Q-function
under the linear mixed-effect model as follows:
\begin{equation*}
Y_t^i = (u^i)^{\top} Z_t^i + \alpha^{\top} X_t^i + \epsilon_t^i, \qquad i\in[N],~ t=0,\dots,T-1,
\end{equation*}
where $Y_t^i$ is the response,
  $(Z_t^i,X_t^i)$ are covariates with subject-specific effect coefficient $u_i$ and fixed
  effect coefficient $\alpha$, and $\epsilon_t^i$ is the error term.
  If we define
\begin{equation*}
  Q(Z_t^i,X_t^i;u^i) = (u^i)^{\top} Z_t^i + \alpha^{\top} X_t^i = \EE \left\{ Y_t^i\given Z_t^i,X_t^i \right\},
\end{equation*}
  then given $u^i$, $Q$ is determined by the shared parameter $\alpha$.
  In our framework, we find $Q$ (the fixed effect) by minimizing the loss
\begin{equation*}
  \sup_{f\in\mF} \sum_{i=1}^N \sum_{t=0}^{T-1} f(Z_t^i,X_t^i;u^i)^{\top} \{Y_t^i - Q(Z_t^i,X_t^i;u^i)\},
\end{equation*}
where the error is captured by $f(\bullet;u^i)$ for each individual. After that, we
update individual latent variable $u^i$ (the random effect), which is included in the policy
optimization step. In our individual policy learning method, the
policy $\pi(\bullet;u^i)$ is estimated simultaneously with $u^i$, $i\in[N]$.
See Section \ref{sec: Pessimistic Policy Learning with Latent Variables} for
more details.

\blue{
In the following, we provide an example of linear MDPs with shared structure.
\begin{example}[Linear MDP \citep{bradtke1996linear,melo2007q}]
  \label{ex: linear MDP}
Suppose all MDPs share the same known bounded measurable feature map
$\bpsi:\mS\times\mA\longrightarrow\RR^d$, with
$\operatorname{span}\{\bpsi(s,a):(s,a)\in\mS\times\mA\}=\RR^d$.
For MDP $\mM^{(i)}$, the transition probability and reward function are
\begin{equation*}
\PP^i(ds_+\given s,a)=\bpsi(s,a)^\top\bmu^i(ds_+),
\qquad r^i(s,a)=\bpsi(s,a)^\top\btheta^i,
\end{equation*}
where $\bmu^i$ is a vector of finite signed measures and the displayed
transition kernel is a probability kernel. We further suppose that
$\bmu^i=\bmu+\bdelta^i$ and $\btheta^i=\btheta+\bbeta^i$, where $\bmu$ and
$\btheta$ are shared for all MDPs, and impose the normalizations
\begin{equation*}
\sup_{s,a}\|\bpsi(s,a)\|_2\leq1,\qquad
\|\btheta^i\|_2\leq\sqrt d,\qquad
\sup_{\|h\|_\infty\leq1}\left\|\int h(s)\bmu^i(ds)\right\|_2\leq\sqrt d.
\end{equation*}
By Lemma \ref{lem: linear MDP} in the Supplementary Material,
\begin{equation*}
Q_{\pi^i}^i(s,a)=\bpsi(s,a)^\top\bw_{\pi^i}^i,
\qquad
\bw_{\pi^i}^i=(I_d-\gamma M_{\pi^i}^i)^{-1}(\btheta+\bbeta^i),
\end{equation*}
where
\begin{equation*}
M_{\pi^i}^i=\int_{\mS}\bmu^i(ds_+)
       \sum_{a\in\mA}\pi^i(a\given s_+)\bpsi(s_+,a)^\top.
\end{equation*}
For a vector of finite signed measures $\boldsymbol\eta$, define the
linear operator on finite signed state--action measures by
\begin{equation*}
[\Gamma_{\boldsymbol\eta}^{\pi}p](ds,a)
=\pi(a\given s)\boldsymbol\eta(ds)^\top
  \int_{\mS\times\mA}\bpsi(s_-,a_-)p(ds_-,da_-).
\end{equation*}
Then, with $(\nu\pi^i)(ds,a)=\nu(ds)\pi^i(a\given s)$, the discounted
occupancy can be represented as
\begin{equation*}
d_{\pi^i}^i=(1-\gamma)\sum_{t=0}^{\infty}\gamma^t
 (\Gamma_{\bmu}^{\pi^i}+\Gamma_{\bdelta^i}^{\pi^i})^t(\nu\pi^i).
\end{equation*}

If $(\bdelta^i,\bbeta^i)$ is constant within each subgroup and belongs to
a fixed finite-dimensional parameter family, we can encode it by
$u^i=(\texttt{para}\{\bdelta^i\},\bbeta^i)$ and set
$\pi^i=\pi(\bullet;u^i)$. The Q-functions then admit a shared representation
indexed by $u^i$. Under within-group coverage, a density-ratio witness has
slice
\begin{equation*}
w_\pi(\bullet;u^{\mG^k})
=\frac{(1-\gamma)\sum_{t=0}^{\infty}\gamma^t
 (\Gamma_{\bmu}^{\pi^k}+\Gamma_{\bdelta^{\mG^k}}^{\pi^k})^t
 (\nu\pi^k)}{\bar d^k},
\end{equation*}
where the quotient denotes a Radon--Nikodym derivative. Boundedness and
simultaneous representation of these slices in $\mF$ are assumed, together
with the regularity conditions in Section \ref{sec: assumptions}.
\end{example}
}

\subsection{Pessimistic Policy Learning with Latent Variables}
\label{sec: Pessimistic Policy Learning with Latent Variables}
Even though we can find the point estimate of $Q_{\pi}$ by solving the empirical min-max
problem \eqref{eq: OPE sample version} for a given $(\pi,\bu)$, it is still
unclear how to quantify
whether the estimated $Q_{\pi}$ could confidently and accurately evaluate the
policy for each individual, partially due to the statistical error and model
misspecification.
For the heterogeneous policy learning goal, to obtain a confident
improvement of the policy, we adopt a pessimistic strategy which
optimizes the policy by using the most pessimistic Q-function estimate $\wt{Q}$
in some uncertainty set defined by
\begin{equation}
\label{eq: uncertainty set}
\Omega(\pi, \bu, \alpha) \triangleq \left\{ Q\in\mQ \colon \max_{f\in\mF}\wh{\Phi}(Q,f,\pi,\bu) \leq \alpha\right\},
\end{equation}
where $\alpha = \alpha_{NT}>0$ is some constant depending on $N$ and $T$.
This ensures
that the lower confidence bound of the values of learned optimal policies is the
largest among all in-class policies.
The construction of the uncertainty set originates from \eqref{eq: OPE sample
  version}.

Then we propose to learn the in-class optimal policy $\pi_{*}$ together with latent
variables $\bu_{*}$ by solving
\begin{equation}
\label{eq: policy learning}
\max_{\pi\in\Pi, \bu} \min_{Q\in\Omega(\pi, \bu, \alpha)} (1-\gamma) \sum_{i=1}^N\EE_{S_0^i\sim\nu} Q(S_0^i,\pi(S_0^i;u^i);u^i),
\end{equation}
which maximizes the policy values of the worst case in the uncertainty sets. We
denote the estimated policy as $\wh{\pi}$ and estimated latent variables as
$\wh{\bu}$.
Throughout, outer maximizations are restricted to candidates with nonempty
uncertainty sets. Through this pessimistic learning approach, we can provide
an overall regret gap with coverage required only for the optimal comparison
policy, under the assumptions in Section \ref{sec: theory}.

Notice that the policy optimization in \eqref{eq: policy learning} can be
difficult due to the constraints on $Q$, especially when the user-defined
function classes are complex. To lessen the computational burden, we use
the Lagrangian dual problem of the primal problem, where the Lagrangian is defined as
\begin{equation}
  \label{eq: lagrangian}
  L(Q,\pi,\bu, \lambda) = (1-\gamma) \sum_{i=1}^N\EE_{S_0^i\sim\nu} Q(S_0^i,\pi(S_0^i;u^i);u^i) + \lambda (\max_{f\in\mF}\wh{\Phi}(Q,f,\pi,\bu) - \alpha).
\end{equation}
Then we can solve the dual problem by
\begin{equation}
\label{eq: dual policy learning}
\max_{\pi\in\Pi, \bu}\max_{\lambda\geq 0} \min_{Q\in\mQ} L(Q,\pi,\bu, \lambda).
\end{equation}
We denote the learned policy and latent variables by $\wh\pi^\dag$ and
$\wh\bu^\dag$, respectively. For each fixed feasible pair $(\pi,\bu)$,
weak duality gives
\[
 \sup_{\lambda\geq0}\min_{Q\in\mQ}L(Q,\pi,\bu,\lambda)
 \leq\min_{Q\in\Omega(\pi,\bu,\alpha)}
       (1-\gamma)\sum_i\EE_\nu Q(S_0,\pi(S_0;u^i);u^i).
\]
The inequality also holds after maximizing over the common feasible outer
domain. In Section \ref{sec: theory}, we establish strong duality in a
strictly feasible neighborhood and show that a local dual solution achieves
the same regret rate as a local primal solution.

In the case when individuals can be classified into several groups, e.g.,
$\left\{ \mG^k \right\}_{k=1}^K$ as a partition of indices
$\left\{ 1,\dots,N \right\}$,
to encourage subgrouping we can impose further constraints on the latent variables
$\bu$, e.g., $\sum_{i\in[N]}\min_{k\in[K]} \norm{ u^i - v^k }_2$, $\rank(\bu) \leq K$ or restricting $\bu = \bv_1\bv_2$ with penalty on $\norm{\bv_1}_F^2+\norm{\bv_2}_F^2$.
In this paper, we study the multi-centroid penalty
$\sum_{i\in[N]}\min_{k\in[K]} \norm{ u^i - v^k }_2^2$, which encourages subgroup
structures by grouping neighboring individuals $u^i$, $i\in\mG^k$ to the nearest
centroid $v^k$ \citep{tang2021individualized}. Unlike the fused-type penalty \citep[see e.g.,][]{chen2022reinforcement}, which requires pairwise
comparison of complexity $\bigO(N^2)$, the multi-centroid penalty significantly
reduces the computational complexity to $\bigO(NK)$, since $K<\!\!< N$ in
general. In addition, the multi-centroid penalty reduces estimation bias,
which is typically large with one-centroid shrinkage penalty, e.g., the $L^2$-penalty.

By adding the multi-centroid penalty to \eqref{eq: policy learning}, we can solve for $(\wh{\pi},\wh{\bu},\wh{\bv})$ through
\begin{equation}
\label{eq: penalized policy learning}
\max_{\pi\in\Pi, \bu,\bv} \min_{Q\in\Omega(\pi, \bu, \alpha)} (1-\gamma) \sum_{i=1}^N\EE_{S_0^i\sim\nu} Q(S_0^i,\pi(S_0^i;u^i);u^i) - \pen_{\mu}(\bu, \bv),
\end{equation}
where
\begin{equation}
\label{eq: penalty}
\pen_{\mu}(\bu,\bv) = \mu\sum_{i\in[N]}\min_{k\in[K]} \left\{ \norm{u^i - v^k}_2^2 \right\}.
\end{equation}
For the guarantees in Section \ref{sec: theory}, we optimize $u^i\in\mU$
and constrain the centroids to the nonempty compact set
\begin{equation}
\label{eq: separated centroids}
\mV_\Delta=\left\{(v^1,\ldots,v^K)\in\mU^K:
  \min_{k\ne\ell}\|v^k-v^\ell\|_2\geq\Delta\right\},
\end{equation}
where $\Delta>0$ is fixed. The separation constraint is vacuous for $K=1$.
Similarly, the penalized dual estimator
$(\wh{\pi}^{\dag}, \wh{\bu}^{\dag}, \wh{\bv}^{\dag})$ can be solved by
\begin{equation}
\label{eq: dual penalized policy learning}
\max_{\pi\in\Pi, \bu, \bv}\max_{\lambda\geq 0} \min_{Q\in\mQ} L(Q,\pi,\bu, \lambda) - \pen_{\mu}(\bu, \bv).
\end{equation}
The dual problem \eqref{eq: dual penalized policy learning} removes the explicit
constraint on $Q$. The computational details are given in
Section \ref{sec: computation}.

\section{Theory}
\label{sec: theory}
In this section, we evaluate the theoretical performance of the proposed policy
optimization method. The performance is often measured by the regret, which is
the value loss of the estimated policy compared to the value of the optimal in-class policy.

With our heterogeneous latent variable model, for any $(\pi,\bu)$, the overall
value is defined as
\begin{equation*}
J(\pi,\bu) = \frac{1}{N} \sum_{i=1}^N J^i\{\pi(\bullet;u^i)\},
\end{equation*}
where each value is evaluated in the actual individual MDP.

Specifically, we study the case when individuals can be partitioned into
$K\geq 1$ groups $\mG^1,\dots,\mG^K$, with corresponding proportions $p_k = |\mG^k|/N$, $k\in[K]$. Here $p_{\min}=\min_k p_k$ is bounded away from zero uniformly in $N$.
Within each subpopulation $\mG^k$, the
individuals share the same latent variable, i.e., $u^i = u^{\mG^k}$ for
$i\in\mG^k$. For the optimal policy $\pi_{*}$, the associated $\bu_{*}$ satisfies
$u_{*}^{\mG^k} \neq u_{*}^{\mG^{k'}}$ when $k \neq k'$.
Write $g(i)=k$ for $i\in\mG^k$, $\bu(\bv)^i=v^{g(i)}$, and
$\mB_{\mG}=\{\bu(\bv):\bv\in\mV_\Delta\}$, with $\mV_\Delta$ as in
\eqref{eq: separated centroids}. The benchmark $(\pi_*,\bu_*)$, where
$\bu_*=\bu(\bv_*)$, maximizes $J$ over $\Pi\times\mB_{\mG}$.
The oracle optimization uses this same domain.

The theoretical challenges of this framework stem from two sources. First, the dimension of
$\bu$ grows linearly with $N$, which causes a large variation of
$\wh{\Phi}(Q,f,\pi,\bu)$ and therefore requires a larger uncertainty set $\Omega(\pi,\bu,\alpha)$ with a
bigger $\alpha$, which is required to capture the true Q-function.
It is crucial that the true Q-function is contained in the proposed uncertainty set,
otherwise the proposed pessimistic learning method would suffer from model misspecification.
However, the increased uncertainty of the true Q-function could further increase the regret
of the estimated optimal policy with latent variables. Therefore, a proper
uncertainty level is required to obtain good individual policies in terms of value. Second, the trajectories
$\left\{ \left( S_t^i,A_t^i,R_t^i \right)_{t=0}^{T-1} \right\}_{i\in[N]}$ are
time-correlated for each individual and their correlation structures can be
varying among different subpopulations. We establish theoretical properties showing
that our method can learn a pair $(\wh{\pi},\wh{\bu})$ well, which can be as good as the oracle pair
$(\pi_{\circ},\bu_{\circ})$ in terms of the regret asymptotically, through carefully choosing
the uncertainty level $\alpha$.

For the rest of this section, we first list several technical assumptions in
Section \ref{sec: assumptions}. In Section \ref{sec: oracle regret and
  consistency}, we derive the overall regret for the oracle policy, when the
true subpopulation information $\left\{ \mG_k \right\}_{k=1}^K$ is known. In
Section \ref{sec: dual regret}, we establish the existence of local penalized
and dual solutions with the oracle regret rate. All proofs of the theorems are provided in the
Supplementary Material.

\subsection{Assumptions}
\label{sec: assumptions}
We first provide several technical assumptions and discuss their implications.
\begin{assumption}
  \label{ass: mixing}
  The individual trajectories are independent. Within each subgroup
  $\mG^k$, individuals share the transition kernel $\PP^k$ and conditional
  mean reward $r^k$, while their behavior policies may differ.
  For each individual $i$, the transition process
  $Z_t^i=(S_t^i,A_t^i,R_t^i,S_{t+1}^i)$ is stationary and exponentially
  $\beta$-mixing, with $\beta_i(j)\leq\beta_0\exp(-\zeta j)$,
  $j\geq1$, for common constants $\beta_0\geq0$ and $0<\zeta\leq1$.
\end{assumption}

Assumption \ref{ass: mixing} characterizes the dependency structure within
individual trajectories. The restriction $\zeta\leq1$ is without loss of
generality, since a smaller exponential mixing rate is also valid.
The dependence between observations decays exponentially with their time
separation, as in \citet{liao2022batch,zhou2022estimating,chen2022reinforcement}.
Within-group homogeneity permits pooling under
$\bar d^k=|\mG^k|^{-1}\sum_{i\in\mG^k}\bar d^i$. Its transition-tuple
expectation is denoted by $\bar\EE^k$.

\begin{definition}[Covering Number]
  \label{def: covering number}
  Let $\left( \mC, \norm{\bullet} \right)$ be a normed space. For any $\mH\subseteq \mC$, the
  set $\left\{ x_i \right\}_{i=1}^n \subseteq \mH$ is a $\epsilon$-covering of $\mH$ if
  $\mH\subseteq \cup_{i=1}^n B(x_i,\epsilon)$, where $B(x,\epsilon)$ is the $\norm{\bullet}$-ball with center
  $x_i$ and radius $\epsilon$. The \textit{covering number} of $\mH$ is defined as
\begin{equation*}
  N(\epsilon,\mH,\norm{\bullet}) = \min \left\{ n\in\NN: \mH\subseteq\cup_{i=1}^nB(x_i,\epsilon) \text{ for some
    } \left\{ x_i \right\}_{i=1}^n \subseteq\mH \right\}.
\end{equation*}
\end{definition}

\begin{assumption}
  \label{ass: spaces}
\begin{enumerate}
  \item [(a)] For any $\mH\in\{\mF,\mQ,\Pi\}$, $\mH$ is a nonempty
        compact class of measurable functions in the uniform norm, and
        $N(\epsilon,\mH,\|\bullet\|_\infty)
        \leq(A_0/\epsilon)^{\fC_{\mH}}$ for $0<\epsilon\leq1$,
        with fixed $A_0\geq2$ and $\fC_{\mH}\geq1$.
  \item [(b)] The latent domain $\mU\subset\RR^{d_u}$ is compact and
        convex, and $\mV_\Delta$ is nonempty for a fixed $\Delta>0$.
  \item [(c)] Uniformly over $Q\in\mQ$, $f\in\mF$, $\pi\in\Pi$, and $s,a$,
\begin{align*}
 |Q(s,a;u)-Q(s,a;u')|&\leq L_Q\|u-u'\|_2,\\
 |f(s,a;u)-f(s,a;u')|&\leq L_f\|u-u'\|_2,\\
 \sum_a|\pi(a\mid s;u)-\pi(a\mid s;u')|
 &\leq L_\pi\|u-u'\|_2.
\end{align*}
  \item [(d)] We have $\sup_{Q\in\mQ}\|Q\|_\infty\leq B=(1-\gamma)^{-1}$
        and $\sup_{f\in\mF}\|f\|_\infty\leq C_{\mF}$.
        The class $\mF$ is symmetric and contains zero.
\end{enumerate}
\end{assumption}

Assumption \ref{ass: spaces} imposes conditions on spaces $\mF$, $\mQ$ and
$\Pi$ to bound their complexities. The polynomial covering condition in
part (a) is common in the literature \citep[e.g.,][]{antos2008learning}.
Bounded parameter classes, including suitable fixed neural network
architectures, provide examples. The Lipschitz conditions control continuity
in the latent variable, while the uniform bounds permit exponential
concentration inequalities \citep{van1996weak}.

\begin{assumption}
  \label{ass: realizability}
  For every $\pi\in\Pi$ and $\bv\in\mV_\Delta$, there exists
  $Q_{\pi,\bv}\in\mQ$ such that
  $Q_{\pi,\bv}(\bullet;v^k)=Q^k_{\pi(\bullet;v^k)}(\bullet)$ for all $k$.
  For the optimal policy $\pi_*^k=\pi_*(\bullet;v_*^k)$, let $d_*^k$ be
  its discounted visitation measure. We have $d_*^k\ll\bar d^k$, and
  there is a single $f_*\in\mF$ satisfying
  $f_*(\bullet;v_*^k)=d d_*^k/d\bar d^k$,
  $\bar d^k$-almost everywhere, simultaneously for all $k\in[K]$.
\end{assumption}

\begin{remark}
Assumption \ref{ass: realizability} requires that the individual Q-functions
can be modeled by the pre-specified class $\mQ$ over the group-constant
representations considered by the oracle. For the optimal policy, the
symmetric class $\mF$ jointly captures the distribution ratios
$d d_*^k/d\bar d^k$. In contrast to per-individual coverage, the mixture of
behavior policies within the same subgroup may cover the relevant
state--action pairs collectively. Realizability and comparator coverage are
commonly imposed for value-based methods
\citep{jiang2020minimax,liao2022batch}.
\end{remark}

\begin{remark}
For the linear MDP in Example \ref{ex: linear MDP}, bounded coefficient
families in the span of $\bpsi$ can supply suitable Q-functions.
The symmetric bounded class $\mF$ must jointly represent the group ratios,
which, when densities exist, have the form
\begin{equation*}
 \frac{1-\gamma}{\bar d^k(s,a)}\sum_{t=0}^{\infty}\gamma^t
  \left(\Gamma_{\bmu}^{\pi_*^k}
              +\Gamma_{\bdelta^{\mG^k}}^{\pi_*^k}\right)^t
              [\nu\pi_*^k](s,a).
\end{equation*}
Under more complex MDP structures, neural networks can be used for $\mQ$
and a similar architecture for $\mF$, subject to the boundedness,
regularity, and realizability conditions above.
\end{remark}

\subsection{Oracle Regret Bounds}
\label{sec: oracle regret and consistency}
In this section, we establish the overall regret bound for the non-penalized
oracle policy with subgroup information. We then establish the same regret
rate for suitable local penalized estimators.

To derive the regret bounds for the oracle policy, we first show that the true
Q-function is a feasible solution under the uncertainty set in \eqref{eq: uncertainty set}
with high probability for a proper choice of $\alpha$.

\begin{theorem}
  \label{thm: feasibility}
  Suppose Assumptions \ref{ass: standard}--\ref{ass: realizability} hold,
  and choose, with a sufficiently large fixed leading constant,
\begin{equation}
  \label{eq: oracle alpha}
 \alpha\asymp\sqrt{
 \frac{\max\{\fC_{\mQ},\fC_{\mF},\fC_\Pi,d_u\}}{NT\zeta}
 \sum_{k=1}^K\log\frac8{p_k\delta}\log(e|\mG^k|T)},
\end{equation}
  where $(NT)^{-2}\leq\delta<1$. With probability at least $1-\delta$,
  simultaneously for every $\pi_\circ\in\Pi$ and
  $\bu_\circ=\bu(\bv_\circ)\in\mB_{\mG}$,
\begin{equation*}
 Q_{\pi_\circ,\bv_\circ}\in\Omega(\pi_\circ,\bu_\circ,\alpha).
\end{equation*}
\end{theorem}

Theorem \ref{thm: feasibility} establishes the foundation of the regret
guarantee for the oracle estimator of the optimal policy with latent variables.
The feasibility of $Q_{\pi_\circ,\bv_\circ}$ in the uncertainty set $\Omega(\pi_{\circ},\bu_{\circ},\alpha)$ by choosing
$\alpha$ in \eqref{eq: oracle alpha} verifies the use of pessimistic policy
learning by solving \eqref{eq: policy learning} with a high probability. That
is, we pessimistically evaluate the value of given oracle pair $(\pi_{\circ},\bu_{\circ})$ by a lower
bound of its true value during the policy optimization. The pessimism in the
offline reinforcement learning can relax the full coverage
assumption by a partial coverage assumption, which only requires that the
offline data cover the trajectory generated by the optimal policy. However, the
realizability assumption for the existence of $Q_{\pi,\bv}\in\mQ$ for all admissible $(\pi,\bv)$ is still
required \citep{jiang2020minimax,xie2021bellman,zhan2022offline}. A recent work
by \citet{chen2023singularity} further relaxes the partial coverage assumption
for policy learning with homogeneous offline data by using the Lebesgue
decomposition theorem and distributional robust optimization, but a
Bellman completeness condition is required in addition. In this paper, we
also impose the partial coverage assumption in Assumption \ref{ass:
  realizability} for technical simplicity.

Equipped with Theorem \ref{thm: feasibility}, we establish the following regret
warranty theorem for the oracle estimator.

\begin{theorem}
  \label{thm: oracle regret}
  Suppose Assumptions \ref{ass: standard} -- \ref{ass: realizability} are
  satisfied, with $\alpha$ and $\delta$ given in Theorem \ref{thm: feasibility}, the
  regret for an oracle maximizer $(\wh{\pi}_{\circ},\wh{\bu}_{\circ})$ of
  \eqref{eq: policy learning} over $\Pi\times\mB_{\mG}$ is bounded by
\begin{equation*}
  J(\pi_{*},\bu_{*}) - J(\wh{\pi}_{\circ}, \wh{\bu}_{\circ}) \lesssim \alpha,
\end{equation*}
with probability at least $1-\delta$.
\end{theorem}

Our regret bound in Theorem \ref{thm: oracle regret} shows that the overall regret of
finding an optimal policy can converge to zero in a rate near $\bigO((NT)^{-1/2})$, as long as the length or trajectory $T$ or sample size $N$ diverges to infinity when the subgroup
information is known. \blue{Up to logarithmic factors, our regret bound of order $\bigO((NT)^{-1/2})$ scales reversely with the number of
individuals $N$ because our algorithm uses all information from $N$
individuals through a shared structure.} However, the method of \citet{chen2022reinforcement} only
establishes the bound for off-policy estimation for a given policy, but not the
regret for a learned optimal policy.

\subsection{Feasibility and Regret Bounds for the Dual Problem}
\label{sec: dual regret}
When the subgroup information is unknown, the optimizer of \eqref{eq: policy
  learning} typically cannot satisfy a regret bound as good as that of the oracle
estimator without constraints on latent variables $\bu$.
However, by solving \eqref{eq: penalized policy learning}, which includes a
multi-centroid penalty on $\bu$, the oracle regret rate is asymptotically achievable
under the regularity and tuning conditions below.

\begin{theorem}
  \label{thm: oracle feasibility}
  Suppose the conditions of Theorem \ref{thm: oracle regret} hold,
  $\mQ$ is convex, and the centroids lie in $\mV_\Delta$.
  Let $\delta=\delta_{NT}\to0$ and $\alpha\to0$, and assume
  $\mu\alpha^2\to\infty$. With probability tending to one, there exists
  a local maximizer $(\wh\pi,\wh\bu,\wh\bv)$ of
  \eqref{eq: penalized policy learning} such that, after indexing its
  centroids by the true subgroups,
\begin{equation}
  \label{eq: latent fusion bound}
  \max_{k\in[K]}\max_{i\in\mG^k}\|\wh u^i-\wh v^k\|_2
      \lesssim(\mu\alpha)^{-1}=o(\alpha).
\end{equation}
  In addition, its regret satisfies
\begin{equation*}
 J(\pi_*,\bu_*)-J(\wh\pi,\wh\bu)=\bigOp(\alpha).
\end{equation*}
\end{theorem}

Theorem \ref{thm: oracle feasibility} shows that a local solution of the
penalized optimization can achieve the oracle regret rate. Even if the
estimated latent variables are not identical within each subgroup, their
within-group dispersion is small enough to preserve this rate. The bound
concerns dispersion around estimated centroids, rather than identification
of particular latent coordinates. By setting $K=1$, the regret has the same
order, up to logarithmic factors, as the homogeneous rates discussed in
\citet{xie2021bellman,zhan2022offline}.

\begin{corollary}
  \label{cor: group oracle regret}
  Suppose the conditions of Theorem \ref{thm: oracle feasibility} hold,
  and the benchmark is individually optimal:
  $J^i\{\pi_*(\bullet;u_*^i)\}\geq J^i\{\pi(\bullet;u)\}$ for every
  $i$, $\pi\in\Pi$, and $u\in\mU$. Define
  $J^{\mG^k}(\pi,\bu)=|\mG^k|^{-1}\sum_{i\in\mG^k}
  J^i\{\pi(\bullet;u^i)\}$.
  For the local solution in Theorem \ref{thm: oracle feasibility},
\begin{equation*}
 J^{\mG^k}(\pi_*,\bu_*)-J^{\mG^k}(\wh\pi,\wh\bu)
       =\bigOp(p_k^{-1}\alpha),\qquad k\in[K].
\end{equation*}
  The same regret order holds for each individual in group $\mG^k$.
\end{corollary}

Corollary \ref{cor: group oracle regret} shows that the rate of the regret for
group $\mG^k$ is near $\bigO_p((Np_k^2T)^{-1/2})$, up to the
logarithmic and complexity factors in $\alpha$, when $\alpha/p_k\to0$. For an individual $i$ in
subgroup $\mG^k$, we can bound its regret through to the proportion of
group $\mG^k$. However, the rate of this subgroup regret bound is slower than
$\bigO_p((Np_kT)^{-1/2})$, which can be obtained by homogeneous RL methods using only data collected from subgroup $\mG^k$, when the subgroup information is
\textit{known}. In our scenario, the subgroup information is \textit{unknown}
and the subgroups with small proportion could be overwhelmed by larger subgroups,
as our optimization objective is based on average regret.

To alleviate the computational burden of solving \eqref{eq: penalized policy learning}, we propose to
solve for a dual problem \eqref{eq: dual penalized policy learning}. In the
following Theorem \ref{thm: dual regret}, we characterize the regret of the dual
local optimizer $(\wh{\pi}^{\dag},\wh{\bu}^{\dag})$ with the convexity of $\mQ$ space.
\begin{theorem}
  \label{thm: dual regret}
  Under the conditions of Theorem \ref{thm: oracle feasibility}, with
  probability tending to one there exists a local maximizer
  $(\wh\pi^\dag,\wh\bu^\dag,\wh\bv^\dag,\wh\lambda^\dag)$ of
  \eqref{eq: dual penalized policy learning} whose regret satisfies
\begin{equation*}
 J(\pi_*,\bu_*)-J(\wh\pi^\dag,\wh\bu^\dag)=\bigOp(\alpha).
\end{equation*}
\end{theorem}

In Theorem \ref{thm: dual regret}, a suitable local dual solution achieves
the same regret rate as the oracle estimator. Convexity and strict
feasibility of the inner Q-function problem ensure strong duality in the
neighborhood used in the proof. Convexity can be restrictive for some
function classes. Extending this result to nonconvex Q-function classes
with an asymptotically negligible duality gap remains of interest.

\section{Computation}
\label{sec: computation}

In this section, we present our computational algorithm to learn
$(\wh{\pi}^{\dag}, \wh{\bu}^{\dag})$ by solving the dual problem \eqref{eq: dual
  penalized policy learning} for multicentroid cases. By choosing $\mQ,\mF$ and $\Pi$ as some
pre-specified functional classes, e.g., neural network (NN) architectures, we can
apply stochastic gradient descent and ascent to solve the dual problem \eqref{eq: dual
  penalized policy learning}, except for $\bu$ and $\bv$.

For the outer maximization problem of \eqref{eq: dual penalized policy learning}, maximizing $\bu$ and $\bv$ with the penalty term $\pen(\bu,\bv)$
is challenging, since the penalty is non-convex with all individualized latent
parameters $\bu$ and subgroup parameters $\bv$ without subgroup information and
their coupling. To achieve a faster computing speed, we propose to apply the
ADMM method \citep{boyd2011distributed} to update $\bu$ and $\bv$ alternatively.
Specifically, we consider solving an equivalent problem given $\pi,\lambda,Q$,
with auxiliary variables $\bw$,
\begin{equation}
  \label{eq: solve uv}
  \begin{aligned}
    & \max_{\bu,\bv,\bw} L(Q,\pi,\bu, \lambda) - \pen_{\mu}(\bw,\bv),\\
    & \text{s.t. } \bw = \bu.
  \end{aligned}
\end{equation}

Applying the augmented Lagrangian method to \eqref{eq: solve uv}, we can solve
\begin{equation}
\label{eq: ALM}
    \max_{\bu,\bv,\bw} L(Q,\pi,\bu, \lambda) - \pen_{\mu}(\bw,\bv) - \boldeta^{\top}(\bu - \bw) - \frac{\rho}{2} \norm{\bu-\bw}_F^2,
\end{equation}
with $\rho>0$ and Lagrangian multiplier vector $\boldeta$. Consequently,
we perform the ADMM method to update $\bu,\bv,\bw$ and $\boldeta$ iteratively:
\begin{equation}
\label{eq: ADMM}
\begin{aligned}
&\bu \leftarrow \argmax_{\bu} L(Q,\pi,\bu, \lambda) -  \frac{\rho}{2} \norm{\bu-\bw + \rho^{-1}\boldeta}_F^2,\\
  &(\bv,\bw) \leftarrow \argmin_{\bv,\bw} \pen_{\mu}(\bw,\bv) + \frac{\rho}{2}\norm{\bu-\bw + \rho^{-1}\boldeta}_F^2,\\
  & \boldeta \leftarrow \boldeta + \rho (\bu -\bw),
\end{aligned}
\end{equation}
where the first block $\bu$ can be approximated by gradient ascent.
$(\bv,\bw)$ is treated as a single block here in the second step.
With unrestricted centers, this block can be approximated by K-means
clustering on $x^i=u^i+\eta^i/\rho$, setting $\bv$ as the cluster means,
then setting $w^i=(\rho x^i+2\mu v^{c_i})/(\rho+2\mu)$, where $c_i$
is the assigned cluster.
For the theoretical parameter space, enforce $\bu\in\mU^N$ and
$\bv\in\mV_\Delta$ in these updates. The regret guarantees concern suitable
local solutions and do not assert convergence of the numerical updates.
We provide the proposed policy optimization algorithm in Algorithm \ref{alg:policyopt}.

With the sub-sampled batch data, we can perform stochastic gradient
  ascent for $f,\pi$ and descent for $Q$, respectively, by
\begin{align}
     f &\leftarrow f + \delta_f \nabla_f \wh{\Phi}(Q,f,\pi,\bu)\label{eq: f update},\\
     Q &\leftarrow Q - \delta_Q \nabla_Q \left\{ (1-\gamma) \sum_{i=1}^N\EE_{S_0^i\sim\nu} Q \left\{ S_0^i,\pi(S_0^i;u^i);u^i \right\} + \lambda \left\{ \wh{\Phi}(Q,f,\pi,\bu) - \alpha \right\} \right\}\label{eq: Q update},\\
     \pi &\leftarrow \pi + \delta_{\pi} \nabla_{\pi} \left\{ (1-\gamma) \sum_{i=1}^N\EE_{S_0^i\sim\nu} Q(S_0^i,\pi(S_0^i;u^i);u^i) + \lambda (\wh{\Phi}(Q,f,\pi,\bu) - \alpha) \right\}.\label{eq: pi update}
\end{align}

Here $[x]_+=\max\{x,0\}$ projects the multiplier ascent onto $[0,\infty)$.

\noindent
\begin{minipage}{\linewidth}
	\begin{algorithm}[H]
    \label{alg:policyopt}
		\SetAlgoLined
		\textbf{Input:} Batch episodes
    $\mD_{NT} = \{(S_t^i,A_t^i,R_t^i,S_{t+1}^i): t=0,...,T-1, i\in[N]\}$. Tuning
    parameters $\alpha,\mu,\rho>0$. Pre-specified NN structures $\mQ,\mF,\Pi$.
    Number of subgroups $K$. Learning rates $\delta_f$, $\delta_Q$, $\delta_{\pi}$, $\delta_{\lambda}$\\

		\textbf{Output:} Estimated individualized policies $(\wh{\pi}^{\dag},\wh{\bu}^{\dag})$. \\

    \textbf{Initialize:} $Q, f, \pi, \bu, \bv$, $\bw=\bu$, $\boldeta=0$, and $\lambda\geq0$.\\
		\textbf{While} $(1-\gamma) \sum_{i=1}^N\EE_{S_0^i\sim\nu} Q(S_0^i,\pi(S_0^i;u^i);u^i)$ is \textbf{not} converged \textbf{do}:\\
		\Indp
    \textbf{Step 1.} Update $f$.\\
    \textbf{While} \textbf{not} converged:\\
    \Indp Sample a size $n_0$ mini-batch $\left\{ (S,A,R,S_+)_i: i\in[n_0] \right\}\sim \mD_{NT}$:\\
     \Indp Update $f$ by \eqref{eq: f update}.\\
     \Indm
     \Indm
    \textbf{Step 2.} Update $Q$.\\
    \textbf{While} \textbf{not} converged:\\
    \Indp Sample a size $n_0$ mini-batch $\left\{ (S,A,R,S_+)_i: i\in[n_0] \right\}\sim \mD_{NT}$:\\
      \Indp Update $Q$ by \eqref{eq: Q update}.\\
		\Indm
		\Indm
    \textbf{Step 3.} Update $\lambda,\pi,\bu,\bv$.\\
    \textbf{While} \textbf{not} converged:\\
    \Indp Sample a size $n_0$ mini-batch $\left\{ (S,A,R,S_+)_i: i\in[n_0] \right\}\sim \mD_{NT}$:\\
      \Indp Update $\pi$ by \eqref{eq: pi update}.\\
		\Indm
		\Indm
     Update $\bu,\bv,\bw,\boldeta$ by \eqref{eq: ADMM}.\\
     Update $\lambda \leftarrow [\lambda + \delta_{\lambda}(\wh{\Phi}(Q,f,\pi,\bu) - \alpha)]_+$.

		\Indm
		\Indm

		\caption{Penalized Pessimistic Personalized Policy Learning (P4L)}%
	\end{algorithm}
\end{minipage}

In practice, we suggest selecting the number of
subgroups by clustering individuals according to their stationary transitions
$\PP^i(s',r\given s,a)$, $i\in[N]$, which are estimated by applying kernel density estimators
on offline trajectories
$\left\{ S_t^i, A_t^i,R_t^i,S_{t+1}^i \right\}_{t=0}^{T-1}$ for each individual $i\in[N]$.

\section{Simulation}
\label{sec: simulation}

In this section, we compare the numerical performance of the proposed method P4L
with several
representative existing works using simulated data from heterogeneous populations.
We first consider a simple environment with a balanced design of behavior policy
on binary action space
in Section \ref{sec: sim1} and demonstrate the numerical performances for
different combinations of $N$ and $T$. Then in Section \ref{sec: sim2}, we
conduct experiments in an OpenAI Gym environment, CartPole with unbalanced
design of behavior policy.

The methods used for comparisons are Fitted-Q-Iteration \citep[FQI,][]{munos2003error,le2019batch},
V-learning \citep[VL,][]{luckett2019estimating}, and Auto-Clustered Policy Iteration \citep[ACPI,][]{chen2022reinforcement}.
The first two methods are based on mean-value models while the last one is
a clustering-based method targeting optimal policies for heterogeneous populations
with different rewards.
For a fair comparison, all methods share the same feature basis for state
$\left\{ \psi_j(s) \right\}_{j=1}^J$ and the same soft-max policy class $\Pi$.

\subsection{A simple environment}
\label{sec: sim1}
Here we briefly describe our simulated data-generating mechanism.
For each individual $i\in[N]$,
the episode starts with a state $S_{0}^i \sim \mN (0, \bI_2)$. Then at each time step
$0\leq t\leq T-1$, the action $A_t^i$ follows a balanced behavior policy
$A_t^i \sim \text{Bernoulli}(1/2)$ on a binary action space $\mA = \{0,1\}$. After
taking action $A_t^i$, the state transits to $S_{t+1}^i$ via
\begin{align*}
        [S_{t+1}^i]_1 &= 0.8(2A_t^i-1) [S_t^i]_1 + c_1^i [S_t^i]_2 + [\epsilon_t^i]_1,\\
        [S_{t+1}^i]_2 &= c_2^i [S_t^i]_1 + 0.8(1-2A_t^i) [S_t^i]_2+ [\epsilon_t^i]_2,
\end{align*}
where $\epsilon_t^i\sim \mN(0, \text{diag}\{0.25, 0.25\})$. At the same time, the
individual receives an immediate reward
$R_t^i = 0.9/\left[ 1+ \exp \left\{ (2A_t^i-1)([S_t^i]_1 - 2[S_t^i]_2) \right\}\right] + \text{Unif}[-0.1,0.1]$.
In addition, we consider a discounted cumulative reward with a discount factor $\gamma=0.8$.
To introduce population heterogeneity, we consider three groups of individuals
with different settings of state transition parameters listed in Table \ref{tab:simu_trans}.

\begin{table}[h]
\caption{State transition parameters for individual $i$ from different groups.}
\label{tab:simu_trans}
  \centering
  \begin{tabular}{c|ccc}
     \toprule
         Parameters & Subgroup (a) & Subgroup (b) & Subgroup (c) \\
     \midrule
         $c_1^i$      & 0         &  0.6      &  -0.7     \\
         $c_2^i$      & -0.6      &  0.4      &  0.5      \\
     \bottomrule
  \end{tabular}
\end{table}

For the P4L, we set rank $r=2,3,4,5$, where $r=3$ is the oracle
number of clusters.
For all methods, the feature mapping basis of the state
$\left\{ \psi_j(s) \right\}_{j=1}^J, J=16$ are Gaussian radial basis functions with
centers selected by K-means and bandwidth selected by the median of pairwise
Euclidean distances on $\left\{ S_t^i:0\leq t\leq T, i\in[N] \right\}$.

\subsubsection{Simulation Results}

We show the results for different combinations of $(n,T)$
for $T\in\{50,100,200\}$ and $|\mG_k|\in\{10, 20, 50\}$ for each group $k\in[3]$, i.e.,
$N\in \left\{ 30, 60 \right\}$ for total sample sizes of individuals.
For each combination of $(n,T)$, we repeat each method 50 times and evaluate the estimated policies by
1000 Monte Carlo trajectories for each group.

\begin{figure}[!htbp]
  \centering
  \includegraphics[width=\textwidth]{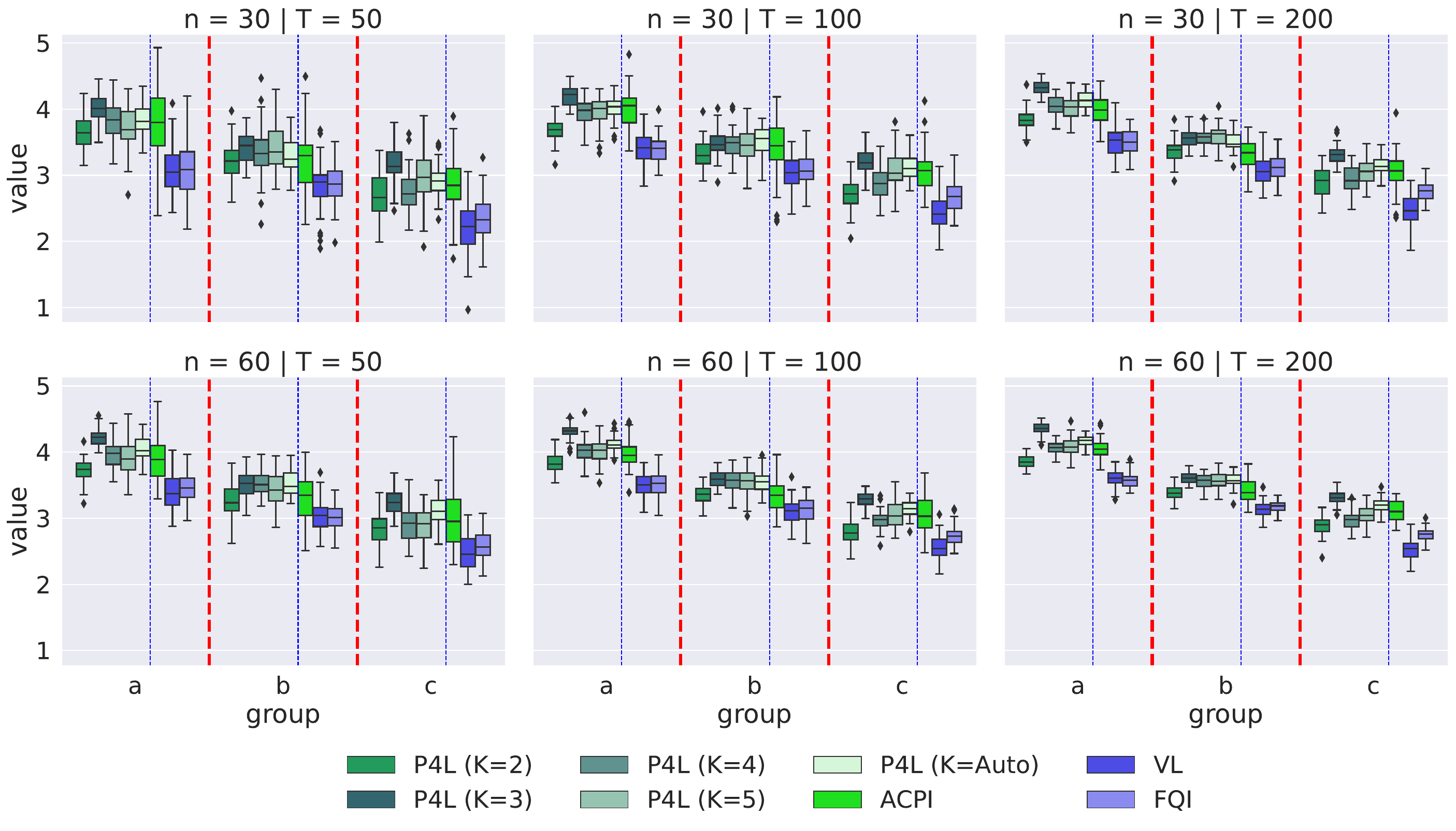}
  \caption{Boxplots for values of estimated policies for $n=30,60$ and
    $T=50,100$ and 200. Red dashed lines separate
    three groups (a), (b) and (c). In each group, the blue dotted line separates
    the P4L method of different $K$ from the benchmark methods.}
  \label{fig:simu_values}
\end{figure}

Figure \ref{fig:simu_values} shows that our method and
ACPI outperform other methods for homogeneous data in
all settings of $(N,T)$.  The VL and FQI aim to find a single optimal policy for heterogeneous
populations based on biased value function estimators, which are not suitable for
 the heterogeneous composition of the environments.
When the number of the
subgroups $K$ is correctly specified, the value of our method is typically
higher than that of ACPI. When the number of
pre-specified subgroups
$K$ is larger than the oracle number, our method is still comparable to
ACPI. When the $K$ is selected by the heuristic method in
Section \ref{sec: computation}, the performance lies between that of the oracle
$K$ and that of larger $K$'s. In addition, the values of
ACPI typically have larger variance than ours since the
information across subgroups is not used. However, when the number of
pre-specified subgroups $K$ is smaller than the truth, our method performs worse
due to the resulting bias under specific $K$.

\subsection{Synthetic Data by OpenAI Gym Environments}
\label{sec: sim2}
We perform experiments on CartPole from the OpenAI Gym, a classical non-linear control
environment. The agent heterogeneity across the various environments is
introduced by modifying the covariates that characterize the transition dynamics
of an agent. This is in line with the existing literature \citep{lee2020context} studying algorithmic robustness of agent heterogeneity.

\begin{figure}[h]
  \centering
  \includegraphics[width=0.3\textwidth]{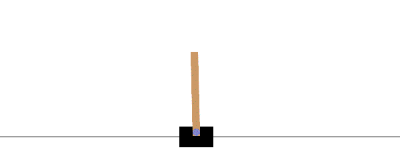}
  \caption{OpenAI CartPole environment.}
  \label{fig:gym_envs}
\end{figure}

In the CartPole environment, a pole is attached to a cart moving on a frictionless
track. The goal is to prevent the pole from falling over by moving the cart to
the left or to the right, and to do so for as long as possible, with a maximum of 300
steps allowed.

The state of the environment is defined by four observations: the cart's
position ($x_t$), and its velocity ($\Dot{x}_t$), and the pole's angle ($\theta_t$) and its
angular velocity ($\Dot{\theta}_t$). Two actions are available to the agent: pushing
the cart to the right (action 1) or to the left (action -1). The behavior policy
is determined by the direction of the pole angle, specifically
$\text{sign}(\theta_t)$. The reward for each step taken without termination is 1. The episode terminates
when the pole angle exceeds 12 degrees or when the cart position exceeds 2.4.
The environment heterogeneity is introduced by varying the length of the pole
and push force within the ranges of $[0.15,0.85]$ and $[2.0,10.0]$,
respectively.

For each setting in each environment, we generate $100$ trajectories of offline
data and evaluate with the starting states of these trajectories. The offline data
is combined with three settings (see Table \ref{tab:performance_openai}) of 300
trajectories in total. For each setting and method, we run simulations 50 times.

\subsubsection{Simulation Results}

Table \ref{tab:performance_openai} shows the performance of our method with
$K=2,3,4,5$ and \emph{Auto} (the heuristic selection method described in Section \ref{sec: computation}) and benchmarks in terms of the playing steps. In settings (A) and
(B), we fix the length or force, and the mixed environments are given by changing
another environment variable. In Settings (C), we mix the changes of both. Overall, in Settings (A)
-- (C), our method outperforms the benchmarks when $K = 3,4,5$ and \emph{Auto} in most settings
of length/force, except $(5/0.5)$ in Setting (B). When $K=2$, the performance of our method is
comparable to that of FQI and VL since they cannot distinguish the
differences among the three heterogeneous environments, since the true number of
environments is three. Compared to our method, ACPI usually suffers from a larger
standard error than ours ($K=3$ and \emph{Auto}), since the policy learning algorithm in ACPI
is performed in each learned cluster of the population. Therefore, the sample
efficiency is typically worse than ours. However, when the number of subgroup $K$
is not correctly prespecified, the standard errors of our method increase,
partially due to the growing number of coefficients with $K$. When $K$ is
heuristically selected, the performance is typically closer to that of the oracle
$K$ than the other choices of $K$,
which confirms the numerical efficacy of our heuristic approach.

\section{Real Data Applications}
\label{sec: data}
\begin{table}[hbtp]
\centering
\caption{Values of learned policies from five RL methods on CartPole with
  different settings. Values in braces are standard errors. In Setting (A), we fix the force at 0.85 and let length
  vary from 2 to 10. In Setting (B), we fix the length at 5 and let force vary
  from 2 to 10. In Setting (C), we mix three different combinations of
  (length/force).}
\label{tab:performance_openai}
{\small
\begin{tabular}{c|ccc}
  \toprule
& \multicolumn{3}{c}{CartPole (length/force)} \\
\midrule
\midrule
Setting (A)  & (2/0.85) & (5/0.85)  & (10/0.85) \\
\midrule
P4L $(K=2)$ & 58.0±12.7           & 72.8±16.8                  & 121.1±31.7 \\
P4L $(K=3)$ & 92.4±17.6           & \textbf{182.8}±4.3         & \textbf{226.0}±12.3\\
P4L $(K=4)$ & 74.3±16.8          & 179.6±8.3                  & 165.5±29.5\\
P4L $(K=5)$ & 86.2±9.3           & 163.2±13.5                  & 192.8±23.0 \\
P4L $(K=Auto)$ & \textbf{98.2}±12.1           & 174.2±6.9                  & 213.0±19.0 \\
\midrule
ACPI       & 89.4±26.5          & 168.2±24.3                  & 200.1±15.0 \\
FQI          & 63.4±17.7           & 65.9±13.7                   & 100.2±11.1\\
VL           & 32.7±3.8           & 71.5±10.1                 & 95.1±10.7\\
\bottomrule
\toprule
Setting (B)  & (5/0.15) & (5/0.5)  & (5/0.85) \\
\midrule
P4L $(K=2)$ & 161.0±21.0           & 142.8±30.1                  & 98.1±13.2 \\
P4L $(K=3)$ & \textbf{192.4}±8.1  & 204.6±8.7                  & \textbf{164.0}±10.2\\
P4L $(K=4)$ & 170.3±16.8          & 211.6±9.3                  & 157.5±9.5\\
P4L $(K=5)$ & 166.2±39.3           & 197.2±13.5                  & 152.8±15.0 \\
P4L $(K=Auto)$ & 174.2±15.7           & 206.0±9.1                  & 154.8±8.9 \\
\midrule
ACPI & 168.4±26.5                & \textbf{216.4}±19.5                  & 151.1±13.9 \\
FQI    & 121.4±17.7                 & 154.9±24.1                   & 122.0±19.1\\
VL     & 146.7±3.8                 & 181.5±10.5                 & 110.1±14.2\\
\bottomrule
  \toprule
Setting (C)  & (2/0.15) & (10/0.85)  & (5/0.5) \\
\midrule
P4L $(K=2)$ & 184.0±22.7           & 190.8±6.8                  & 158.1±10.7 \\
P4L $(K=3)$ & \textbf{236.4}±7.3  & \textbf{203.8}±4.1         & \textbf{192.0}±2.2\\
P4L $(K=4)$ & 165.3±16.8          & 179.6±8.3                  & 127.5±9.5\\
P4L $(K=5)$ & 166.2±39.3           & 181.2±13.5                  & 182.8±15.0 \\
P4L $(K=Auto)$ & 181.2±9.1           & 199.2±7.7                  & 146.8±7.4 \\
\midrule
ACPI & 165.4±21.5                & 193.2±19.5                  & 132.1±15.0 \\
FQI    & 154.4±17.7                 & 164.9±8.7                   & 132.2±13.6\\
VL     & 133.7±3.8                 & 127.5±20.1                 & 110.1±10.7\\
\bottomrule
\end{tabular}
}
\end{table}
In this section, we use the Multi-parameter Intelligent Monitoring in Intensive
Care (MIMIC-III)
dataset (\url{https://physionet.org/content/mimiciii/1.4/}) to evaluate
the performance of the learned optimal policies. In this dataset,
longitudinal information (e.g., demographics, vitals, labs and standard scores)
 was collected from 17,621 patients who satisfied
the SEPSIS criterion from 5 ICUs at a Boston teaching hospital. See detailed information in \citet{nanayakkara2022unifying}.

The dataset is prepared by the same data pre-processing steps used in
\citet{raghu2017deep}. For the cleaned dataset,
the state space consists of a 41-dimensional real-valued vector, which includes
information on patients' demographics, vitals, standard scores and lab results.
A detailed description is given in Supplementary Material. Specifically, the
most important standard score, which captures the patient's organ function, is
the Sequential Organ Failure Assessment (SOFA) score, which is
commonly used in clinical practice to assess SEPSIS severity.
In the preprocessed dataset for each patient, the state records are aligned on
a 4-hour time grid for simplicity. We only consider patients who had at least 3
time stamps, which results in 16,356 patients.
In addition, due to the high dimensionality of the state space, we perform a
dimension reduction procedure via principle component analysis.
Specifically, we select the top 10 principal components that explain
at least 95\% of the total variation of the state, resulting in an
11-dimensional encoded state that consists of 10 principal components and the
SOFA score.

At each time point $t$ after the enrollment, the patient is treated with
vasopressor and/or intravenous fluid, or no treatment. To ensure each action has considerable
times of repeating in the dataset, we discretize the level of two treatments
into 3 bins, respectively. The combination of the two drugs results in an action
space of size $|\mA|=9$. We use the negative SOFA score at time $t+1$ as the
immediate reward at time $t$, i.e., $R_t=-s_{t+1}^{\text{SOFA}}$, as it
provides a direct and clinically meaningful measure of the patient's SEPSIS
condition. Moreover, such a reward function has been used in recent literature \citep[see,
e.g.,][]{wang2022blessing}. Throughout this section, the discount factor of the reward is chosen to be $\gamma=0.8$.

We apply the proposed method P4L with a number of subgroups selected by the proposed heuristic
method, and a Gaussian RBF basis
with bandwidth selected by the median of pairwise Euclidean distances of states.
For the competing methods, we use the same basis. The same settings are used in Section \ref{sec: simulation}.

Since the outcomes of optimal personalized policies are not available in the dataset, to evaluate the numerical performance
in terms of value, we learn the personalized transitions for all
patients in the preprocessed data set by PerSim
\citep{agarwal2021persim}, a state-of-the-art method to estimate personalized
simulators. Then we use the estimated simulator and reward function to evaluate
the values of learned policies by the Monte Carlo method.
In the processed dataset, we randomly subsample 1000 ICU admissions, apply the
P4L method and benchmarks, and evaluate the estimated policies, as well as
the clinicians' decisions, by the learned PerSim
simulator. We repeat this procedure 5 times with different random seeds
and demonstrate the aggregated estimated values.

Based on Figure \ref{fig:realdata} and Table \ref{tab:realdata}, it is clear that
our method outperforms the others in terms of value (accumulative discounted negative SOFA
scores). This indicates that our approach leads to better decisions in the
treatment of SEPSIS. We note that the values of clinicians' decisions are second
to that of our method, and with the lowest variance, since there is no
change of measure due to the mismatch of behavior policy and target policy. The
values obtained by VL, and FQI are relatively low, suggesting that these methods
do not adequately account for the heterogeneity in the population of patients.
ACPI performs slightly worse than
clinician's decisions and has a larger variance, partially due to the reduction in
sample efficiency from policy learning after clustering. In contrast, our method
can effectively use the samples and learn heterogeneous treatment policies for
SEPSIS patients in the MIMIC-3 study.

\begin{table}[h]
  \caption{Values of learning optimal policies with standard errors in parentheses.}
  \label{tab:realdata}
  \centering
  \begin{tabular}{c|ccccc}
    \toprule
    Method & Clinician & P4L & ACPI & VL & FQI\\
    \midrule
    Value  & -6.13 (1.82) & -5.40 (2.10) & -6.54 (2.31) & -7.32 (2.55) & -7.09 (2.65)\\
    \bottomrule
  \end{tabular}
\end{table}

\begin{figure}[h]
  \centering
	\includegraphics[width=0.7\textwidth]{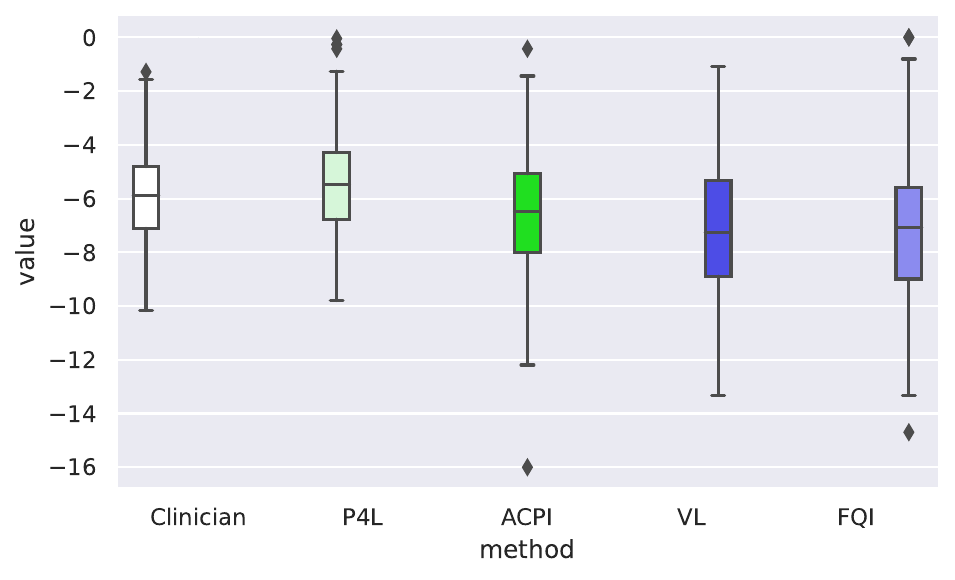}
  \caption{Values of learned optimal policies estimated by PerSim.}
  \label{fig:realdata}
\end{figure}

\section{Discussion}
\label{sec: discussion}

In this paper we introduce a novel RL framework for offline policy
optimization using batch data from heterogeneous time-stationary MDPs.
Our approach incorporates individualized latent variables into a shared heterogeneous
model, facilitating efficient estimation of individual Q-functions. Following
this framework, we propose a pessimistic policy learning method admitting local solutions with a fast
average regret rate under within-group partial coverage and the regularity conditions
in Section \ref{sec: theory}.

Our theoretical properties on the regret of the proposed method rely on the
partial coverage assumption that the mixture of behavior policies within each
homogeneous subgroup adequately covers its optimal policy. This coverage assumption can
potentially be relaxed by applying Lebesgue decomposition and distribution
robust optimization \citep{chen2023singularity}, which requires
an additional Bellman completeness assumption. Another potential theoretical
improvement is to relax the convexity assumption on the $\mQ$ space for the dual
problem such that the duality gap is asymptotically negligible.
In addition, the regret bound in our approach is not guaranteed to be tight, as our method focuses on the learning rate rather than the lower bound. Compared to the existing literature on offline RL for homogeneous data under similar data coverage assumptions, our theoretical regret rate for homogeneous data matches that of \citet{xie2021bellman} and
\citet{zhan2022offline}.
Further investigation into the theoretical challenges of establishing regret lower bounds for offline RL with heterogeneous data presents an interesting direction for future research.

The proposed policy optimization method can be applied to dynamic decision-making problems where the state-action transitions are stationary but
vary among individuals, and the agents continue making decisions in
their environments with the updated policies. Examples include mobile health for
patients with chronic diseases and robotics working in
different slowly-changing environments. For a new individual not in the
pre-collected data, one needs to obtain more informative state-action transitions and
classify the individual to a subgroup, then apply the policy of that group.
This can be explored in future work.

In this work, we have focused on resolving populational heterogeneity for time-stationary MDPs.
However, it is important to acknowledge that in practice, the assumptions of
time-stationarity may not hold for episodic MDPs.
To address this issue, \citet{hu2022doubly} try to learn on optimal
policies for the last period of the stationary part, but this cannot guarantee long-term
optimality. Future research could involve exploring the policy learning for time-nonstationary
heterogeneous episodic MDPs.
Moreover, it is also worthwhile to explore individual policy optimization when
confronted with unmeasured confoundings and alternative types of rewards, such
as time to event. These aspects remain largely unexplored in the existing
literature for heterogeneous data and present interesting future research directions.

\bibliographystyle{abbrvnat}
\bibliography{reference}

\newpage
{\centering
\Large \bf Supplementary Material to ``Reinforcement Learning for Individual Optimal Policy
from Heterogeneous Data''\\
}
\bigskip

\renewcommand\thesection{S.\arabic{section}}
\setcounter{section}{0}

\section{Proofs of Theoretical Results}

\subsection{Continuity bounds}
We first derive the continuity estimates used in the proofs. All constants
in this subsection depend only on the bounds in Assumption
\ref{ass: spaces} and on $\gamma$. In particular, they do not depend on
$N,T$ or on the observed data. Set $B=(1-\gamma)^{-1}$ and define
\begin{equation}
\label{eq: continuity constants}
\begin{aligned}
 L_\Phi&=2B L_f+C_{\mF}\{(1+\gamma)L_Q+\gamma B L_\pi\},\\
 L_V&=(1-\gamma)(L_Q+B L_\pi),\qquad L_J=B L_\pi.
\end{aligned}
\end{equation}

\textit{Action-averaged critics.}
Fix $Q\in\mQ$ and $\pi\in\Pi$. For $u\in\mU$, write
\[
 W_u(s)=\sum_{a\in\mA}\pi(a\mid s;u)Q(s,a;u).
\]
Since the policy probabilities sum to one, $|W_u(s)|\leq B$.
For $d=\|u-u'\|_2$, add and subtract
$\sum_a\pi(a\mid s;u)Q(s,a;u')$ to obtain
\begin{align*}
 |W_u(s)-W_{u'}(s)|
 &\leq\sum_a\pi(a\mid s;u)|Q(s,a;u)-Q(s,a;u')|\\
 &\quad+\sum_a|\pi(a\mid s;u)-\pi(a\mid s;u')|\,|Q(s,a;u')|\\
 &\leq (L_Q+B L_\pi)d.
\end{align*}
The estimate is uniform in $s,Q,\pi$.

\textit{Empirical Bellman moments.}
For a transition tuple $z=(s,a,r,s_+)$ with $|r|\leq1$, define
\[
 b_u(z)=r+\gamma W_u(s_+)-Q(s,a;u),\qquad
 h_u(z)=f(s,a;u)b_u(z),\qquad f\in\mF.
\]
The reward normalization and the definition of $B$ give
\[
 |b_u(z)|\leq1+(1+\gamma)B=2B.
\]
The preceding bound for $W_u$ also gives
\[
 |b_u(z)-b_{u'}(z)|
 \leq\{(1+\gamma)L_Q+\gamma B L_\pi\}d.
\]
We next add and subtract $f(s,a;u')b_u(z)$. It follows that
\begin{align*}
 |h_u(z)-h_{u'}(z)|
 &\leq|f(s,a;u)-f(s,a;u')|\,|b_u(z)|\\
 &\quad+|f(s,a;u')|\,|b_u(z)-b_{u'}(z)|\\
 &\leq L_\Phi d.
\end{align*}
Apply this pointwise inequality to every observed tuple, sum over $i,t$,
and divide by $NT$. For any $\bu,\bu'\in\mU^N$,
\begin{equation}
 \label{eq: empirical latent continuity}
 |\wh\Phi(Q,f,\pi,\bu)-\wh\Phi(Q,f,\pi,\bu')|
 \leq\frac{L_\Phi}{N}\sum_{i=1}^N\|u^i-u'^i\|_2.
\end{equation}
Taking expectations gives the same bound for $\Phi$.
For any two bounded families $a_f,b_f$,
$|\sup_f a_f-\sup_f b_f|\leq\sup_f|a_f-b_f|$.
Consequently \eqref{eq: empirical latent continuity} also holds after
taking the supremum over $f$ separately in the two empirical moments.

\textit{Initial-state objectives.}
The normalized initial-state objective associated with $Q$ is
\[
 \frac{1-\gamma}{N}\sum_i\EE_{S_0\sim\nu}W_{u^i}(S_0).
\]
The bound for $W_u-W_{u'}$ implies that changing $\bu$ to $\bu'$
changes this objective by at most
\[
 \frac{L_V}{N}\sum_i\|u^i-u'^i\|_2.
\]
For the unnormalized sum used in \eqref{eq: policy learning}, the bound
is $L_V\sum_i\|u^i-u'^i\|_2$. In particular, changing only coordinate
$i$ by distance $d$ changes that sum by at most $L_Vd$.

\textit{Actual policy values.}
These bounds on model objectives do not by themselves compare actual
policy values at different latent coordinates. To obtain the latter
comparison, fix one individual MDP and let $V^\rho$ denote its true
unnormalized state value under a stationary policy $\rho$. Bounded rewards
give $\|V^\rho\|_\infty\leq B$. For two policies $\rho,\rho'$, subtract
their Bellman equations and insert the intermediate term obtained by
using $\rho$ with $V^{\rho'}$. For every state $s$,
\begin{align*}
 V^\rho(s)-V^{\rho'}(s)
 &=\gamma\sum_a\rho(a\mid s)
       \int\{V^\rho(s_+)-V^{\rho'}(s_+)\}\PP(ds_+\mid s,a)\\
 &\quad+\sum_a\{\rho(a\mid s)-\rho'(a\mid s)\}
       \left\{r(s,a)+\gamma\int V^{\rho'}(s_+)\PP(ds_+\mid s,a)\right\}.
\end{align*}
The expression in the second pair of braces is bounded by
$1+\gamma B=B$. If
$\eta=\sup_s\sum_a|\rho(a\mid s)-\rho'(a\mid s)|$, then
\[
 \|V^\rho-V^{\rho'}\|_\infty
 \leq\gamma\|V^\rho-V^{\rho'}\|_\infty+B\eta.
\]
Rearranging and multiplying by $1-\gamma$ gives
\[
 |J^i(\rho)-J^i(\rho')|
 \leq(1-\gamma)\|V^\rho-V^{\rho'}\|_\infty
 \leq B\eta.
\]
Take $\rho=\pi(\bullet;u)$ and $\rho'=\pi(\bullet;u')$.
Assumption \ref{ass: spaces}(c) implies $\eta\leq L_\pi\|u-u'\|_2$.
Thus, for every individual $i$,
\begin{equation}
 \label{eq: individual value continuity}
 |J^i\{\pi(\bullet;u)\}-J^i\{\pi(\bullet;u')\}|
 \leq L_J\|u-u'\|_2.
\end{equation}
Averaging over individuals proves
$|J(\pi,\bu)-J(\pi,\bu')|
 \leq L_JN^{-1}\sum_i\|u^i-u'^i\|_2$.
This comparison always uses the same actual MDP on both sides for a given
individual. It does not assume realizability at the changed coordinate.

\subsection{Proof of Lemma \ref{lem: OPE error}}
\begin{proof}
Fix an individual $i$, a policy $\pi^i$, and a bounded measurable
candidate $\wt Q^i$. Put $M=\|\wt Q^i\|_\infty$ and define
\[
 \wt V^i(s)=\sum_{a\in\mA}\pi^i(a\mid s)\wt Q^i(s,a).
\]
Then $\|\wt V^i\|_\infty\leq M$. All expectations in this proof
refer to a trajectory of the actual MDP of individual $i$, initialized
from $\nu$ and following $\pi^i$.

For a nonnegative integer $m$, consider the finite weighted sum
\[
 D_m=(1-\gamma)\sum_{t=0}^m\gamma^t
 \EE^{\pi^i}\{R_t^i+\gamma\wt V^i(S_{t+1}^i)
                           -\wt Q^i(S_t^i,A_t^i)\}.
\]
Conditioning on $S_t^i$ and using the policy to average over $A_t^i$
gives $\EE^{\pi^i}\wt Q^i(S_t^i,A_t^i)
=\EE^{\pi^i}\wt V^i(S_t^i)$. Hence the finite sums telescope exactly:
\begin{align*}
 D_m
 &=(1-\gamma)\sum_{t=0}^m\gamma^t\EE^{\pi^i}R_t^i\\
 &\quad+(1-\gamma)\left\{
       \sum_{t=1}^{m+1}\gamma^t\EE^{\pi^i}\wt V^i(S_t^i)
       -\sum_{t=0}^{m}\gamma^t\EE^{\pi^i}\wt V^i(S_t^i)\right\}\\
 &=(1-\gamma)\sum_{t=0}^m\gamma^t\EE^{\pi^i}R_t^i
       -(1-\gamma)\EE_\nu\wt V^i(S_0^i)\\
 &\quad+(1-\gamma)\gamma^{m+1}
                       \EE^{\pi^i}\wt V^i(S_{m+1}^i).
\end{align*}
The absolute value of the last term is at most
$(1-\gamma)\gamma^{m+1}M$, which tends to zero. Bounded rewards imply
absolute convergence of the reward series. The definition of the true
Q-function and the initial action distribution give
\[
 (1-\gamma)\sum_{t=0}^{\infty}\gamma^t\EE^{\pi^i}R_t^i
 =(1-\gamma)\EE_\nu Q_{\pi^i}^i(S_0^i,\pi^i(S_0^i))
 =J^i(\pi^i).
\]
The other initial-state term is $\wt J^i(\pi^i)$. Thus
$\lim_{m\to\infty}D_m=J^i(\pi^i)-\wt J^i(\pi^i)$.

It remains to identify this limit with the stated visitation expectation.
Define the bounded conditional residual
\[
 b^i(s,a)=r^i(s,a)+\gamma\int\wt V^i(s_+)\PP^i(ds_+\mid s,a)
                       -\wt Q^i(s,a).
\]
It satisfies $\|b^i\|_\infty\leq1+(1+\gamma)M$.
The conditional reward and transition assumptions imply
\[
 D_m=(1-\gamma)\sum_{t=0}^m\gamma^t
             \int b^i(s,a)\,p_{\pi^i,t}^i(ds,da).
\]
Because the sum of the weights is finite and $b^i$ is bounded, the
integral can be interchanged with the infinite series. By the definition
of $d_{\pi^i}^i$, the limit is
$\int b^i(s,a)\,d_{\pi^i}^i(ds,da)$.
Expanding $b^i$ as the conditional expectation of the sampled residual
proves \eqref{eq: OPE error}. This calculation also applies when
$\gamma=0$ and does not require visitation measures to possess densities.
\end{proof}

\subsection{Proof of Theorem \ref{thm: feasibility}}
\begin{proof}
Write $n_k=|\mG^k|$, so $p_k=n_k/N$, and set
$B=(1-\gamma)^{-1}$. For $z=(s,a,r,s_+)\in\mS\times\mA\times[-1,1]\times\mS$, define
\[
 b_{Q,\pi,v}(z)
 =r+\gamma\sum_{a'\in\mA}Q(s_+,a';v)\pi(a'\mid s_+;v)-Q(s,a;v),
 \qquad
 h_{Q,f,\pi,v}(z)=f(s,a;v)b_{Q,\pi,v}(z).
\]
Let $\mathcal H$ contain these functions as $(Q,f,\pi,v)$ ranges over
$\mQ\times\mF\times\Pi\times\mU$. We first verify the boundedness and
covering conditions needed for Lemma \ref{lem:hoeffding-mixing}.
Since $|r|\leq1$ and $\sum_{a'}\pi(a'\mid s_+;v)=1$,
\[
 |b_{Q,\pi,v}(z)|\leq1+\gamma B+B=2B,
 \qquad |h_{Q,f,\pi,v}(z)|\leq2BC_{\mF}.
\]
Use the strictly positive envelope $B_h=1+2BC_{\mF}$.

For clarity, the uniform norm on policies is
$\|\pi-\pi'\|_\infty
 =\sup_{s,a,v}|\pi(a\mid s;v)-\pi'(a\mid s;v)|$.
At a fixed latent argument, adding and subtracting
$Q'(s_+,a';v)\pi(a'\mid s_+;v)$ gives
\begin{align*}
 &\left|\sum_{a'}Q(s_+,a';v)\pi(a'\mid s_+;v)
           -\sum_{a'}Q'(s_+,a';v)\pi'(a'\mid s_+;v)\right|\\
 &\qquad\leq\|Q-Q'\|_\infty
          +B\sum_{a'}|\pi(a'\mid s_+;v)-\pi'(a'\mid s_+;v)|\\
 &\qquad\leq\|Q-Q'\|_\infty
          +B|\mA|\|\pi-\pi'\|_\infty.
\end{align*}
Thus, by changing $f,Q,\pi$, and then $v$ one at a time,
\begin{align*}
 \|h_{Q,f,\pi,v}-h_{Q',f',\pi',v'}\|_\infty
 &\leq 2B\|f-f'\|_\infty
      +C_{\mF}(1+\gamma)\|Q-Q'\|_\infty\\
 &\quad+\gamma BC_{\mF}|\mA|\|\pi-\pi'\|_\infty
      +L_\Phi\|v-v'\|_2.
\end{align*}
The last term uses \eqref{eq: continuity constants}. More explicitly,
changing $v$ changes the residual by at most
$\{(1+\gamma)L_Q+\gamma B L_\pi\}\|v-v'\|_2$.
Changing the multiplying function $f$ contributes at most
$2B L_f\|v-v'\|_2$, which gives the stated $L_\Phi$.

Let $R_{\mU}$ be the diameter of the compact set $\mU$ and put
\[
 L_h=2B+C_{\mF}(1+\gamma)+\gamma BC_{\mF}|\mA|+L_\Phi,
 \qquad M=\max\{1,L_h/B_h\}.
\]
For $0<\epsilon\leq1$, take uniform $\epsilon/M$-nets for
$\mQ,\mF,\Pi$ and a Euclidean $\epsilon/M$-net for $\mU$.
The product net induces an $\epsilon B_h$-net for $\mathcal H$ by the
preceding inequality. A maximal $\eta$-separated subset of $\mU$
has at most $(1+2R_{\mU}/\eta)^{d_u}$ points. Indeed, disjoint
balls of radius $\eta/2$ about its points lie in a ball of radius
$R_{\mU}+\eta/2$, and comparison of their Euclidean volumes gives
this bound. Maximality makes the subset an $\eta$-net. It follows that
\begin{align*}
 N(\epsilon B_h,\mathcal H,\|\bullet\|_\infty)
 &\leq\left(\frac{A_0M}{\epsilon}\right)^{
          \fC_{\mQ}+\fC_{\mF}+\fC_\Pi}
       \left(1+\frac{2R_{\mU}M}{\epsilon}\right)^{d_u}\\
 &\leq(A/\epsilon)^D,
 \qquad D=\fC_{\mQ}+\fC_{\mF}+\fC_\Pi+d_u,
\end{align*}
where $A=\max\{2,A_0M,(1+2R_{\mU})M\}$ is fixed.
These finite nets also give a countable dense subset of $\mathcal H$,
so all the suprema below are measurable.

Define the population and empirical group moments by
\begin{align}
 \Psi_k(Q,f,\pi,v)&=\frac1{n_k}\sum_{i\in\mG^k}
                           \EE h_{Q,f,\pi,v}(Z_0^i),\label{eq: Psi}\\
 \wh\Psi_k(Q,f,\pi,v)&=\frac1{n_kT}\sum_{i\in\mG^k}
                  \sum_{t=0}^{T-1}h_{Q,f,\pi,v}(Z_t^i).
                  \label{eq: Psihat}
\end{align}
Stationarity identifies the expectation in \eqref{eq: Psi} with the
expectation of \eqref{eq: Psihat}. For a group-constant latent vector,
regrouping the sums in the definition of $\wh\Phi$ gives
\begin{equation}
\label{eq: population Phi}
 \Phi(Q,f,\pi,\bu(\bv))=\sum_k p_k\Psi_k(Q,f,\pi,v^k),
 \qquad \wh\Phi=\sum_k p_k\wh\Psi_k.
\end{equation}
Here $\Phi$ is the expectation of $\wh\Phi$ for fixed arguments.
Within-group homogeneity also identifies $\Psi_k$ with
$\bar\EE^k h_{Q,f,\pi,v}$, since the residual depends on the reward
only through its conditional mean.

For each group, apply Lemma \ref{lem:hoeffding-mixing} with
$N$ replaced by $n_k$ and failure probability $\delta p_k$.
Assumption \ref{ass: mixing} supplies independent trajectories and the
common mixing constants. The block length is the same for every group,
since
\[
 \frac{n_kT}{\delta p_k}=\frac{NT}{\delta},
 \qquad
 a=\min\left\{T,\max\left\{1,
    \left\lceil\zeta^{-1}
       \log\frac{8(\beta_0\vee1)NT}{\delta}\right\rceil
                   \right\}\right\}.
\]
Write $\mathcal E_k$ for the event on which
\[
 \sup_{Q,f,\pi,v}|\wh\Psi_k(Q,f,\pi,v)-\Psi_k(Q,f,\pi,v)|
 \leq C\sqrt{\frac{a}{n_kT}\,c_k},
 \qquad c_k=D\log(2A)+\log\frac8{\delta p_k}.
\]
The fixed envelope $B_h$ has been absorbed into $C$.
Since $\Pr(\mathcal E_k^c)\leq\delta p_k$ and $\sum_kp_k=1$,
\[
 \Pr\left(\bigcap_{k=1}^K\mathcal E_k\right)
 \geq1-\sum_{k=1}^K\Pr(\mathcal E_k^c)
 \geq1-\delta.
\]
No independence between these events is required. On their intersection,
the triangle inequality followed by Cauchy--Schwarz gives
\begin{equation}
\label{eq: Th1prf1}
\begin{aligned}
 \sup_{Q,f,\pi,\bv}|\wh\Phi-\Phi|
 &\leq\sum_kp_k\sup_{Q,f,\pi,v}|\wh\Psi_k-\Psi_k|\\
 &\leq C\sqrt{\frac a{NT}}\sum_k\sqrt{p_kc_k}\\
 &\leq C\sqrt{\frac a{NT}}
           \left(\sum_kp_k\right)^{1/2}
           \left(\sum_kc_k\right)^{1/2}\\
 &=C\sqrt{\frac a{NT}
       \left\{KD\log(2A)+\sum_k\log\frac8{\delta p_k}\right\}}.
\end{aligned}
\end{equation}
The second line uses $p_k/\sqrt{n_k}=\sqrt{p_k/N}$.
In the first supremum, $\bv\in\mV_\Delta$. Bounding each of its
coordinates by the supremum over all $v\in\mU$ is valid even though
$\mV_\Delta$ imposes separation between coordinates.

We now compare \eqref{eq: Th1prf1} with the stated choice of $\alpha$.
Put $C_* =\max\{\fC_{\mQ},\fC_{\mF},\fC_\Pi,d_u\}$, so
$D\leq4C_*$ and $C_*\geq1$. The restriction
$\delta\geq(NT)^{-2}$ gives
\[
 \log\frac{8(\beta_0\vee1)NT}{\delta}
 \leq\log\{8(\beta_0\vee1)\}+3\log(NT).
\]
Therefore, the ceiling in the definition of $a$ and $\zeta\leq1$
imply $a\leq C\zeta^{-1}\log(eNT)$. Also,
$\log\{8/(\delta p_k)\}\geq\log8$ implies
\[
 c_k\leq\left\{1+\frac{4\log(2A)}{\log8}\right\}
           C_*\log\frac8{\delta p_k}.
\]
Finally, $n_k=p_kN$ and $\log(e n_kT)\geq1$ yield
\begin{align*}
 \log(eNT)
 &=\log(e n_kT)+\log(1/p_k)\\
 &\leq\{1+\log(1/p_{\min})\}\log(e n_kT).
\end{align*}
The assumed fixed lower bound on $p_{\min}$ makes the multiplier fixed.
Combining these three bounds with \eqref{eq: Th1prf1} gives
\[
 \sup_{Q,f,\pi,\bv}|\wh\Phi-\Phi|
 \leq C\sqrt{\frac{C_*}{NT\zeta}
              \sum_{k=1}^K\log\frac8{\delta p_k}\log(e n_kT)}
 \leq\alpha/4,
\]
where the last inequality follows by choosing the fixed leading constant
in \eqref{eq: oracle alpha} sufficiently large. Thus, on a single event
of probability at least $1-\delta$,
\begin{equation}
 \label{eq: uniform oracle event}
 \sup_{Q,f,\pi,\bv}
 \left|\wh\Phi(Q,f,\pi,\bu(\bv))-
              \Phi(Q,f,\pi,\bu(\bv))\right|\leq\alpha/4,
\end{equation}
where $Q\in\mQ$, $f\in\mF$, $\pi\in\Pi$, and $\bv\in\mV_\Delta$.

It remains to verify that the population moment vanishes for a true
Q-function. Fix any admissible $(\pi,\bv)$ and write
$\pi^k=\pi(\bullet;v^k)$. Realizability gives
$Q_{\pi,\bv}(\bullet;v^k)=Q^k_{\pi^k}$. For each $i\in\mG^k$,
the conditional Bellman equation is
\begin{align*}
 &\EE\left[b_{Q_{\pi,\bv},\pi,v^k}(Z_t^i)
                                  \mid S_t^i=s,A_t^i=a\right]\\
 &\quad=r^k(s,a)+\gamma\int_{\mS}\sum_{a'}
       Q^k_{\pi^k}(s_+,a')\pi^k(a'\mid s_+)
                         \PP^k(ds_+\mid s,a)-Q^k_{\pi^k}(s,a)=0.
\end{align*}
Because $f(S_t^i,A_t^i;v^k)$ is measurable with respect to
$(S_t^i,A_t^i)$, taking conditional expectations gives
$\EE h_{Q_{\pi,\bv},f,\pi,v^k}(Z_t^i)=0$ for every $f\in\mF$.
Averaging first within groups and then across groups shows
$\Phi(Q_{\pi,\bv},f,\pi,\bu(\bv))=0$. Consequently, on
\eqref{eq: uniform oracle event},
\[
 \sup_{f\in\mF}\wh\Phi(Q_{\pi,\bv},f,\pi,\bu(\bv))
 \leq\alpha/4<\alpha.
\]
This is exactly the feasibility condition in
\eqref{eq: uncertainty set}. The event is uniform over all admissible
$(\pi,\bv)$, so it also covers pairs selected using the data.
\end{proof}

\subsection{Proof of Theorem \ref{thm: oracle regret}}
\begin{proof}
Define the normalized initial-state objective and its pessimistic value by
\[
 V(Q,\pi,\bu)=\frac{1-\gamma}{N}\sum_{i=1}^N
        \EE_\nu\sum_{a\in\mA}Q(S_0,a;u^i)\pi(a\mid S_0;u^i),
 \qquad
 \wh V(\pi,\bu)=\min_{Q\in\Omega(\pi,\bu,\alpha)}V(Q,\pi,\bu).
\]
The objective in \eqref{eq: policy learning} is $N\wh V(\pi,\bu)$.
Multiplication by the positive constant $N$ does not change its
maximizers, so the oracle estimator also maximizes $\wh V$ over
$\Pi\times\mB_{\mG}$.

Work throughout on the event \eqref{eq: uniform oracle event}, whose
probability is at least $1-\delta$. For every admissible $(\pi,\bv)$,
the preceding proof gives
$Q_{\pi,\bv}\in\Omega(\pi,\bu(\bv),\alpha)$. Its slices are the
true Q-functions in each subgroup. The definition of $J$ therefore gives
\[
 \wh V(\pi,\bu(\bv))
 \leq V(Q_{\pi,\bv},\pi,\bu(\bv))
 =\frac1N\sum_iJ^i\{\pi(\bullet;v^{g(i)})\}
 =J(\pi,\bu(\bv)).
\]
This inequality applies in particular to the data-dependent oracle
estimator because the event is uniform.

We next bound the pessimism at the fixed comparator $(\pi_*,\bu_*)$.
For an arbitrary $Q\in\mQ$, write
$Q^k(s,a)=Q(s,a;v_*^k)$, $\pi_*^k=\pi_*(\bullet;v_*^k)$, and define
its conditional Bellman residual in group $k$ as
\[
 \bar b_Q^k(s,a)=r^k(s,a)
  +\gamma\int_{\mS}\sum_{a'}Q^k(s_+,a')\pi_*^k(a'\mid s_+)
                           \PP^k(ds_+\mid s,a)-Q^k(s,a).
\]
Let $J^k(\pi_*^k)$ denote the common value in group $k$, and put
$V_Q^k=(1-\gamma)\EE_\nu\sum_aQ^k(S_0,a)\pi_*^k(a\mid S_0)$.
The OPE identity in Lemma \ref{lem: OPE error} gives
\[
 J^k(\pi_*^k)-V_Q^k=\int\bar b_Q^k(s,a)\,d_*^k(ds,da).
\]
Both $d_*^k$ and $J^k$ use the normalization factor $1-\gamma$.
Thus there is no additional factor of $(1-\gamma)^{-1}$ in this identity.
By Assumption \ref{ass: realizability}, a single $f_*\in\mF$
simultaneously represents all group density ratios. The change of measure
and then conditional expectation give
\begin{align*}
 \int\bar b_Q^k\,d_*^k
 &=\int f_*(s,a;v_*^k)\bar b_Q^k(s,a)\,\bar d^k(ds,da)\\
 &=\bar\EE^k\left[f_*(S,A;v_*^k)
       \left\{R+\gamma\sum_{a'}Q^k(S_+,a')\pi_*^k(a'\mid S_+)
                    -Q^k(S,A)\right\}\right]\\
 &=\Psi_k(Q,f_*,\pi_*,v_*^k).
\end{align*}
All these integrals are finite because the residual and $f_*$ are bounded.
The conditional transition kernel and mean reward in this calculation
are common within group $k$. The different behavior policies enter only
through $\bar d^k$. Multiplying by $p_k$ and summing now yields
\begin{align*}
 J(\pi_*,\bu_*)-V(Q,\pi_*,\bu_*)
 &=\sum_kp_k\{J^k(\pi_*^k)-V_Q^k\}\\
 &=\sum_kp_k\Psi_k(Q,f_*,\pi_*,v_*^k)
 =\Phi(Q,f_*,\pi_*,\bu_*).
\end{align*}
Using one simultaneous witness is what permits the last expression to be
controlled by the single supremum over $\mF$ in the uncertainty set.

If $Q\in\Omega(\pi_*,\bu_*,\alpha)$, the uniform deviation bound and
feasibility give, respectively,
\begin{align*}
 J(\pi_*,\bu_*)-V(Q,\pi_*,\bu_*)
 &\leq\wh\Phi(Q,f_*,\pi_*,\bu_*)+\alpha/4\\
 &\leq\sup_{f\in\mF}\wh\Phi(Q,f,\pi_*,\bu_*)+\alpha/4
 \leq5\alpha/4.
\end{align*}
This holds for every feasible $Q$. Taking the minimum of
$V(Q,\pi_*,\bu_*)$ over that set therefore gives
$\wh V(\pi_*,\bu_*)\geq J(\pi_*,\bu_*)-5\alpha/4$.
Combining pessimism for the estimated pair, optimality of that pair,
and the comparator bound, in that order, gives
\[
 J(\wh\pi_\circ,\wh\bu_\circ)
 \geq\wh V(\wh\pi_\circ,\wh\bu_\circ)
 \geq\wh V(\pi_*,\bu_*)
 \geq J(\pi_*,\bu_*)-5\alpha/4.
\]
This proves the claimed regret bound on the same event, with no further
probability loss.

For completeness, the inner minimum is well defined and attained.
For fixed $(\pi,\bu)$, the function
$Q\mapsto\sup_f\wh\Phi(Q,f,\pi,\bu)$ is Lipschitz with constant
$C_{\mF}(1+\gamma)$ in the uniform norm, by the residual calculation in
the preceding proof. Its sublevel set $\Omega$ is therefore closed in
the compact class $\mQ$. It is nonempty on the oracle domain by
feasibility. The continuous objective $V(Q,\pi,\bu)$ attains its minimum
on this compact set.
For a nonconvex $\mQ$, the theorem takes an outer maximizer as given.
If only an $\eta$-optimal pair for the normalized objective is used,
replace the middle inequality in the last display by
$\wh V(\wh\pi_\circ,\wh\bu_\circ)
 \geq\wh V(\pi_*,\bu_*)-\eta$.
The resulting regret bound is $5\alpha/4+\eta$.
Such approximate maximizers exist for every $\eta>0$ because the
nonempty oracle domain has a bounded profiled objective.
When $\mQ$ is convex, strict feasibility and the mixture argument in the
next proof also give attainment of the outer maximum on the compact
oracle domain.
\end{proof}

\subsection{Proof of Theorem \ref{thm: oracle feasibility}}
\begin{proof}
Write $g(i)=k$ when $i\in\mG^k$, and, for $\bv\in\mV_\Delta$,
let $\bu(\bv)^i=v^{g(i)}$. Define
\begin{align*}
  \ell(Q,\pi,\bu)
    &=(1-\gamma)\sum_{i=1}^N
          \EE_{S_0^i\sim\nu}Q(S_0^i,\pi(S_0^i;u^i);u^i),\\
  \widehat c(Q,\pi,\bu)&=\sup_{f\in\mF}
                                  \wh\Phi(Q,f,\pi,\bu),\\
  V(\pi,\bu)&=\min_{Q\in\Omega(\pi,\bu,\alpha)}
                                   \ell(Q,\pi,\bu),\\
  H(\pi,\bu,\bv)&=V(\pi,\bu)-\pen_\mu(\bu,\bv).
\end{align*}
We prove below that the minimum defining $V$ exists throughout the
neighborhood used in the argument. Since $\|Q\|_\infty\leq B$ and
$(1-\gamma)B=1$, we always have $|\ell(Q,\pi,\bu)|\leq N$.

\paragraph{Step 1: A compact neighborhood with strict feasibility.}
Let $E_{NT}$ be the uniform concentration event used in
Theorems \ref{thm: feasibility} and \ref{thm: oracle regret}.
Its probability is at least $1-\delta$. Work on this event for now.
For every group-constant pair $(\pi,\bu(\bv))$, the Bellman equation
for its realizing function and the uniform concentration bound give
\begin{equation}
  \label{eq: oracle strict feasibility}
  \widehat c(Q_{\pi,\bv},\pi,\bu(\bv))\leq\alpha/4.
\end{equation}
The Bellman-witness argument at the benchmark, applied to every feasible
$Q$ before taking the infimum, also gives
\begin{equation}
  \label{eq: benchmark pessimism}
  V(\pi_*,\bu_*)\geq N\{J(\pi_*,\bu_*)-5\alpha/4\}.
\end{equation}

Use the continuity constants in \eqref{eq: continuity constants} and set
\begin{align*}
 r&=\min\left\{\frac{\Delta}{8},
                    \frac{\alpha}{4(1+L_\Phi)}\right\},\\
 \mathcal T_r&=\left\{(\pi,\bu,\bv)\in
    \Pi\times\mU^N\times\mV_\Delta:
    \max_i\|u^i-v^{g(i)}\|_2\leq r\right\}.
\end{align*}
We use the uniform-norm topology on $\Pi$ and the Euclidean topology on
the latent coordinates. The centroid set $\mV_\Delta$ is closed in the
compact set $\mU^K$, because each separation constraint is closed.
Thus $\Pi\times\mU^N\times\mV_\Delta$ is compact. Each tube inequality
is also closed, so $\mathcal T_r$ is compact. It is nonempty because it
contains $(\pi,\bu(\bv),\bv)$ for any $\pi\in\Pi$ and
$\bv\in\mV_\Delta$.

For a point of $\mathcal T_r$, write $d_i=\|u^i-v^{g(i)}\|_2$ and
$\bar d=N^{-1}\sum_i d_i$. If $k\ne g(i)$, the triangle inequality gives
\[
 \|u^i-v^k\|_2\geq\|v^k-v^{g(i)}\|_2-d_i
       \geq\Delta-r\geq7\Delta/8>r\geq d_i.
\]
Consequently $v^{g(i)}$ is the unique nearest centroid to $u^i$.
For $K=1$ this conclusion holds without a separation argument. In either
case the penalty is exactly
\begin{equation}
  \label{eq: tube penalty}
  \pen_\mu(\bu,\bv)=\mu\sum_i d_i^2
       \quad\text{on }\mathcal T_r.
\end{equation}
For every fixed $Q$, the empirical-moment continuity bound holds
uniformly over $f$, and hence also holds after taking its supremum.
Apply it to $Q_{\pi,\bv}$ and use
\eqref{eq: oracle strict feasibility} to obtain
\begin{equation}
  \label{eq: tube strict feasibility}
  \widehat c(Q_{\pi,\bv},\pi,\bu)
      \leq\alpha/4+L_\Phi\bar d
      \leq\alpha/4+L_\Phi r\leq\alpha/2.
\end{equation}
Thus every point of the tube has a feasible function with constraint
margin at least $\alpha/2$. We have used realizability only at the
corresponding group-constant embedding $\bu(\bv)$.

\paragraph{Step 2: Continuity and attainment of the tube maximum.}
For fixed $\pi,\bu$, the map $Q\mapsto\ell(Q,\pi,\bu)$ is affine.
The empirical moment is affine in $Q$ for each $f$, so its supremum
$Q\mapsto\widehat c(Q,\pi,\bu)$ is convex. To verify joint continuity,
write $\overline d(\bu,\bu')=N^{-1}\sum_i\|u^i-u'^i\|_2$.
The finite action space, bounded function classes, and the continuity
bounds give
\begin{align*}
 |\ell(Q,\pi,\bu)-\ell(Q',\pi',\bu')|
 &\leq N(1-\gamma)\{
       \|Q-Q'\|_\infty+B|\mA|\|\pi-\pi'\|_\infty\}
       +NL_V\overline d(\bu,\bu'),\\
 |\widehat c(Q,\pi,\bu)-\widehat c(Q',\pi',\bu')|
 &\leq C_{\mF}\{
       (1+\gamma)\|Q-Q'\|_\infty
       +\gamma B|\mA|\|\pi-\pi'\|_\infty\}
       +L_\Phi\overline d(\bu,\bu').
\end{align*}
For the second inequality, bound the difference of empirical moments
for each common $f$ and use
$|\sup_f a_f-\sup_f b_f|\leq\sup_f|a_f-b_f|$.
These estimates prove the required continuity. The feasible set in
$\mQ$ is closed, nonempty by \eqref{eq: tube strict feasibility}, and
compact. A minimizing sequence for $\ell$ therefore has a convergent
subsequence in this feasible set, whose limit attains $V$.

Here is a sequential proof that $V$ is continuous on the tube.
Let $x_n=(\pi_n,\bu_n,\bv_n)\to x=(\pi,\bu,\bv)$ in
$\mathcal T_r$, and abbreviate $V(x)=V(\pi,\bu)$.
Since $-N\leq V(x_n)\leq N$, choose a subsequence on which $V(x_n)$
converges to its limit inferior. Choose an inner minimizer $Q_n$ on
this subsequence. Compactness of $\mQ$ gives a further subsequence
with $Q_n\to Q$ uniformly. Continuity implies
\[
 \widehat c(Q,\pi,\bu)
   =\lim_n\widehat c(Q_n,\pi_n,\bu_n)\leq\alpha,
 \qquad
 \lim_n V(x_n)=\ell(Q,\pi,\bu)\geq V(x).
\]
This proves $\liminf_nV(x_n)\geq V(x)$.

For the reverse inequality, fix an inner minimizer $Q^*$ at $x$ and
let $Q^0=Q_{\pi,\bv}$. For $0<\eta<1$, convexity of $\mQ$ gives
$Q^\eta=(1-\eta)Q^*+\eta Q^0\in\mQ$. Constraint convexity gives
\[
 \widehat c(Q^\eta,\pi,\bu)
  \leq(1-\eta)\alpha+\eta\alpha/2
  =\alpha-\eta\alpha/2.
\]
For this fixed $\eta$, continuity makes $Q^\eta$ feasible at $x_n$
for all sufficiently large $n$. Hence
\begin{align*}
 \limsup_nV(x_n)
 &\leq\lim_n\ell(Q^\eta,\pi_n,\bu_n)
   =\ell(Q^\eta,\pi,\bu)\\
 &=(1-\eta)V(x)+\eta\ell(Q^0,\pi,\bu)
   \leq V(x)+2N\eta.
\end{align*}
Letting $\eta\downarrow0$ proves $\limsup_nV(x_n)\leq V(x)$.
Thus $V$ is continuous. The penalty is a finite sum of minima of
continuous functions, so $H$ is continuous as well. A maximizing
sequence for $H$ on the compact tube has a convergent subsequence,
and continuity makes its limit a maximizer.

The maximizer can be selected measurably. To see the measurability of
the profiled objective, take a countable dense subset
$\{Q_j:j\geq1\}$ of the compact metric class $\mQ$. On $E_{NT}$,
\[
 V(x)=\inf_{j:\,\widehat c(Q_j,\pi,\bu)<\alpha}
                          \ell(Q_j,\pi,\bu).
\]
Every function in this infimum is feasible. Conversely, mix a primal
minimizer with a strictly feasible function as above, then approximate
the mixture uniformly by the $Q_j$. Continuity preserves strict
feasibility for sufficiently close approximations, and their objectives
converge to the mixture objective. Letting the mixture weight tend to
zero proves the equality. For fixed $x$, the supremum defining
$\widehat c(Q_j,\pi,\bu)$ equals its supremum over a countable dense
subset of $\mF$, since the moment is continuous in $f$. It is therefore
measurable in the data. The countable infimum makes $V(x)$ and $H(x)$
measurable as well. Extend $H$ to be identically zero outside $E_{NT}$.
This extension is measurable in the data for each $x$ and continuous
in $x$ for every data realization on the fixed compact metric tube.
The measurable maximum theorem then supplies a measurable tube
maximizer. Outside $E_{NT}$, select any fixed point of the tube.

\paragraph{Step 3: The value change from moving one coordinate.}
Fix $\pi,\bv$ and two points $\bu,\bu'$ in the tube that differ only
in coordinate $i$. Put $d=\|u^i-u'^i\|_2$ and
$\beta=L_\Phi d/N$. Let $Q$ attain $V(\pi,\bu)$, and choose the
strictly feasible function $Q^0=Q_{\pi,\bv}$ at $(\pi,\bu')$.
Then
\[
 \widehat c(Q,\pi,\bu')\leq\alpha+\beta,
 \qquad \widehat c(Q^0,\pi,\bu')\leq\alpha/2.
\]
Set
\[
 Q^\theta=(1-\theta)Q+\theta Q^0,
 \qquad \theta=\frac{\beta}{\alpha/2+\beta}\in[0,1).
\]
Convexity gives $Q^\theta\in\mQ$, and the particular choice of
$\theta$ gives the exact constraint calculation
\[
 \widehat c(Q^\theta,\pi,\bu')
 \leq(1-\theta)(\alpha+\beta)+\theta\alpha/2
 =\alpha+\beta-\theta(\alpha/2+\beta)=\alpha.
\]
The objective change from replacing $Q$ by $Q^\theta$ at $\bu'$ is
at most $2N\theta$, because both unmixed objectives lie in $[-N,N]$.
The change from $\bu$ to $\bu'$ with $Q$ fixed is at most $L_Vd$.
Consequently
\begin{align*}
 V(\pi,\bu')
 &\leq\ell(Q^\theta,\pi,\bu')\\
 &\leq\ell(Q,\pi,\bu')+2N\theta\\
 &\leq V(\pi,\bu)+L_Vd+2N\theta\\
 &\leq V(\pi,\bu)+\left(L_V+\frac{4L_\Phi}{\alpha}\right)d,
\end{align*}
where $\theta\leq2\beta/\alpha$ was used in the last line.
Starting with a minimizer at $\bu'$ and the strictly feasible function
at $\bu$ gives
$V(\pi,\bu)\leq V(\pi,\bu')+(L_V+4L_\Phi/\alpha)d$.
Combining these two inequalities yields
\begin{equation}
 \label{eq: tube value continuity}
 |V(\pi,\bu')-V(\pi,\bu)|\leq M_\alpha\|u^i-u'^i\|_2,
 \qquad M_\alpha=L_V+4L_\Phi/\alpha.
\end{equation}

\paragraph{Step 4: Dispersion and local optimality.}
Let $(\wh\pi,\wh\bu,\wh\bv)$ maximize $H$ on $\mathcal T_r$.
For a fixed $i$, define $\bu^{(i)}$ by replacing $\wh u^i$ with
$\wh v^{g(i)}$ and leaving every other coordinate fixed.
This replacement remains in $\mU^N$ and in the tube. If
$d_i=\|\wh u^i-\wh v^{g(i)}\|_2$, \eqref{eq: tube penalty} gives
\[
 \pen_\mu(\bu^{(i)},\wh\bv)
   =\pen_\mu(\wh\bu,\wh\bv)-\mu d_i^2.
\]
Writing out tube optimality for this particular competitor gives
\begin{align*}
 V(\wh\pi,\wh\bu)-\pen_\mu(\wh\bu,\wh\bv)
 &\geq V(\wh\pi,\bu^{(i)})
       -\pen_\mu(\wh\bu,\wh\bv)+\mu d_i^2,\\
 \mu d_i^2
 &\leq V(\wh\pi,\wh\bu)-V(\wh\pi,\bu^{(i)})
 \leq M_\alpha d_i.
\end{align*}
For $d_i>0$, division by $\mu d_i$ gives $d_i\leq M_\alpha/\mu$.
The same inequality holds when $d_i=0$. Therefore
\begin{equation}
  \label{eq: tube distance bound}
  \varepsilon_{NT}:=\max_i\|\wh u^i-\wh v^{g(i)}\|_2
       \leq M_\alpha/\mu
       \lesssim(\mu\alpha)^{-1}=o(\alpha).
\end{equation}
Indeed, for $\alpha\leq1$,
$M_\alpha\leq(L_V+4L_\Phi)/\alpha$, and
$(\mu\alpha)^{-1}/\alpha=(\mu\alpha^2)^{-1}\to0$.

Because $\alpha\to0$ and $\Delta>0$ is fixed, eventually
$r=\alpha/\{4(1+L_\Phi)\}$. More explicitly,
\[
 \frac{M_\alpha}{\mu r}
  =\frac{4(1+L_\Phi)(L_V\alpha+4L_\Phi)}{\mu\alpha^2}
       \longrightarrow0.
\]
Thus $\varepsilon_{NT}<r$ for all sufficiently large sample sizes.
This strict inequality gives a neighborhood even when the centroids
are allowed to move. Set $\rho=(r-\varepsilon_{NT})/4>0$.
Any point of the original parameter domain with
\[
 \max_i\|u^i-\wh u^i\|_2<\rho,
 \qquad \max_k\|v^k-\wh v^k\|_2<\rho
\]
satisfies, for every $i$,
\[
 \|u^i-v^{g(i)}\|_2
 \leq\|u^i-\wh u^i\|_2
       +\|\wh u^i-\wh v^{g(i)}\|_2
       +\|\wh v^{g(i)}-v^{g(i)}\|_2
 <\varepsilon_{NT}+2\rho<r.
\]
These inequalities define a relative open neighborhood of
$(\wh\pi,\wh\bu,\wh\bv)$ contained in $\mathcal T_r$.
Tube optimality therefore implies local optimality for
\eqref{eq: penalized policy learning}, relative to its stipulated
domain $\Pi\times\mU^N\times\mV_\Delta$.
This argument does not require a centroid to be interior to
$\mV_\Delta$. The true groups are used to construct the neighborhood
in the proof, and do not enter the optimization problem.

\paragraph{Step 5: Both sides of the regret bound.}
Let $\bar\bu=\bu(\wh\bv)$ and
$\bar d=N^{-1}\sum_i\|\wh u^i-\wh v^{g(i)}\|_2
\leq\varepsilon_{NT}$. Strict feasibility gives
$Q_{\wh\pi,\wh\bv}\in\Omega(\wh\pi,\wh\bu,\alpha)$.
At the group-constant embedding $\bar\bu$, realizability implies
\[
 N^{-1}\ell(Q_{\wh\pi,\wh\bv},\wh\pi,\bar\bu)
       =J(\wh\pi,\bar\bu).
\]
Using this equality and the two continuity bounds in succession gives
\begin{align*}
 N^{-1}V(\wh\pi,\wh\bu)
 &\leq N^{-1}\ell(Q_{\wh\pi,\wh\bv},\wh\pi,\wh\bu)\\
 &\leq J(\wh\pi,\bar\bu)+L_V\bar d\\
 &\leq J(\wh\pi,\wh\bu)+(L_V+L_J)\bar d.
\end{align*}
The benchmark triple $(\pi_*,\bu_*,\bv_*)$ lies in the tube and has
zero penalty. Hence tube optimality and
\eqref{eq: benchmark pessimism} give
\begin{align*}
 V(\wh\pi,\wh\bu)-\pen_\mu(\wh\bu,\wh\bv)
 &\geq V(\pi_*,\bu_*)\\
 &\geq N\{J(\pi_*,\bu_*)-5\alpha/4\}.
\end{align*}
Combining the last two displays and rearranging gives
\begin{align*}
 J(\pi_*,\bu_*)-J(\wh\pi,\wh\bu)
 &\leq5\alpha/4+(L_V+L_J)\bar d
             -N^{-1}\pen_\mu(\wh\bu,\wh\bv)\\
 &\leq5\alpha/4+(L_V+L_J)\varepsilon_{NT}.
\end{align*}
For the lower bound, $(\wh\pi,\bar\bu)$ belongs to the benchmark's
group-constant domain. Its optimality therefore gives
$J(\pi_*,\bu_*)\geq J(\wh\pi,\bar\bu)$, and consequently
\[
 J(\pi_*,\bu_*)-J(\wh\pi,\wh\bu)
 \geq J(\wh\pi,\bar\bu)-J(\wh\pi,\wh\bu)
 \geq-L_J\bar d\geq-L_J\varepsilon_{NT}.
\]
By \eqref{eq: tube distance bound}, eventually $\varepsilon_{NT}\leq\alpha$.
Thus the absolute regret is at most $C_1\alpha$ on $E_{NT}$ with
$C_1=5/4+L_V+L_J$. Since $\Pr(E_{NT}^c)\leq\delta_{NT}\to0$, the constructed
solution exists with probability tending to one, and its absolute regret
exceeds $C_1\alpha$ with probability at most $\delta_{NT}$.
On the exceptional event the selected parameters are the fixed
admissible triple specified above. This does not affect the probability
bound and proves the claimed
$\bigOp(\alpha)$ regret and \eqref{eq: latent fusion bound}.
\end{proof}

\subsection{Proof of Corollary \ref{cor: group oracle regret}}
\begin{proof}
Define the individual and group losses by
\begin{align*}
 R_i&=J^i\{\pi_*(\bullet;u_*^i)\}
               -J^i\{\wh\pi(\bullet;\wh u^i)\},\\
 R_k^{\mG}&=\frac1{|\mG^k|}\sum_{i\in\mG^k}R_i,
 \qquad R=\frac1N\sum_{i=1}^N R_i.
\end{align*}
Individual optimality implies $R_i\geq0$ for every $i$, so
$R_k^{\mG}\geq0$. Partitioning the individual sum by groups yields
$R=\sum_{k=1}^K p_kR_k^{\mG}$. On the event in the preceding proof,
$R\leq C_1\alpha$ eventually. For every $k$ this gives
\[
 0\leq p_kR_k^{\mG}\leq\sum_{\ell=1}^Kp_\ell R_\ell^{\mG}
       =R\leq C_1\alpha,
 \qquad 0\leq R_k^{\mG}\leq C_1p_k^{-1}\alpha.
\]
The event has probability tending to one, which proves the group claim.

To obtain the individual claim without a factor $|\mG^k|$, use the
small dispersion of the fitted latent coordinates. Within group $k$,
all individuals have the same MDP and initial law $\nu$. Let $W_k$
be their common value when the policy is
$\wh\pi(\bullet;\wh v^k)$. For each $i\in\mG^k$, the individual
version of the value continuity bound gives
\[
 \left|J^i\{\wh\pi(\bullet;\wh u^i)\}-W_k\right|
     \leq L_J\|\wh u^i-\wh v^k\|_2
     \leq L_J\varepsilon_{NT}.
\]
Averaging this inequality over the group gives
\[
 |J^{\mG^k}(\wh\pi,\wh\bu)-W_k|
 \leq\frac1{|\mG^k|}\sum_{j\in\mG^k}
       \left|J^j\{\wh\pi(\bullet;\wh u^j)\}-W_k\right|
 \leq L_J\varepsilon_{NT}.
\]
The triangle inequality therefore yields
\[
 \left|J^i\{\wh\pi(\bullet;\wh u^i)\}
                  -J^{\mG^k}(\wh\pi,\wh\bu)\right|
 \leq2L_J\varepsilon_{NT}.
\]
The benchmark values agree within the group because their MDPs,
initial law, and benchmark latent coordinate $v_*^k$ agree. Subtracting
the definitions of $R_i$ and $R_k^{\mG}$ thus gives
\[
 0\leq R_i\leq R_k^{\mG}+2L_J\varepsilon_{NT}
       \leq C_1p_k^{-1}\alpha+2L_J\varepsilon_{NT}.
\]
Since $\varepsilon_{NT}=o(\alpha)$ on this event and $p_k\leq1$,
the last expression is bounded by $C_2p_k^{-1}\alpha$ eventually,
where $C_2=C_1+2L_J$ is independent of $i,k,N,T$. The probability of
the complementary event tends to zero, proving the individual claim.
\end{proof}

\subsection{Proof of Theorem \ref{thm: dual regret}}
\begin{proof}
Work on the event and tube $\mathcal T_r$ constructed in the proof of
Theorem \ref{thm: oracle feasibility}. For a fixed triple
$(\pi,\bu,\bv)\in\mathcal T_r$, define
\begin{align*}
 q_{\pi,\bu}(\lambda)
 &=\min_{Q\in\mQ}\{\ell(Q,\pi,\bu)
                  +\lambda[\widehat c(Q,\pi,\bu)-\alpha]\},\\
 D(\pi,\bu)&=\sup_{\lambda\geq0}q_{\pi,\bu}(\lambda).
\end{align*}
Each minimum exists because its objective is continuous and $\mQ$
is compact. We give a direct scalar proof of strong duality, including
the existence of a finite maximizing multiplier.

\paragraph{Step 1: The perturbation value and a supporting slope.}
Suppress $\pi,\bu$ temporarily and write
$g(Q)=\widehat c(Q,\pi,\bu)-\alpha$. Set
\[
 h(t)=\inf\{\ell(Q,\pi,\bu):Q\in\mQ,\ g(Q)\leq t\},
\]
with $h(t)=+\infty$ when the feasible set is empty.
Continuity and compactness give $m=\min_{Q\in\mQ}g(Q)$.
The effective domain of $h$ is exactly $[m,\infty)$, and on this
domain $-N\leq h(t)\leq N$ and the infimum is attained.
Strict feasibility in \eqref{eq: tube strict feasibility} implies
$m\leq-\alpha/2$. In particular, $h$ is finite on a neighborhood
of zero and $h(0)=V(\pi,\bu)$.

If $t_1\leq t_2$, every function feasible at $t_1$ is feasible at
$t_2$, so $h(t_2)\leq h(t_1)$. To prove convexity, take
$t_1,t_2\geq m$ and their minimizing functions $Q_1,Q_2$.
For $0\leq\theta\leq1$, the function
$Q_\theta=(1-\theta)Q_1+\theta Q_2$ lies in $\mQ$ and satisfies
\begin{align*}
 g(Q_\theta)&\leq(1-\theta)g(Q_1)+\theta g(Q_2)
               \leq(1-\theta)t_1+\theta t_2,\\
 h((1-\theta)t_1+\theta t_2)
 &\leq\ell(Q_\theta,\pi,\bu)
    =(1-\theta)h(t_1)+\theta h(t_2).
\end{align*}
If either $h(t_j)=+\infty$, the extended-real convexity inequality
is automatic for $0<\theta<1$. Thus $h$ is convex and nonincreasing.

We construct a finite supporting slope explicitly. Put $a=\alpha/4$.
For $m\leq t<0<v$, convexity at
$0=\{v/(v-t)\}t+\{-t/(v-t)\}v$ gives
\[
 \frac{h(0)-h(t)}{-t}\leq\frac{h(v)-h(0)}{v}.
\]
Define
\[
 s=\sup_{m\leq t<0}\frac{h(0)-h(t)}{-t}.
\]
The set in this supremum is nonempty, since $m\leq-2a$.
Taking $t=-a$ for a lower bound and $v=a$ in the preceding
secant inequality for an upper bound gives
\[
 -\infty<\frac{h(0)-h(-a)}a
       \leq s\leq\frac{h(a)-h(0)}a\leq0.
\]
Thus $s$ is finite and nonpositive. For $m\leq t<0$, the definition
of $s$ implies $h(t)\geq h(0)+st$. For $t>0$, the secant inequality
with every negative argument implies
$s\leq\{h(t)-h(0)\}/t$, which gives the same bound.
At $t=0$ equality holds, and for $t<m$ the bound is automatic because
$h(t)=+\infty$. We have therefore proved
\[
 h(t)\geq h(0)+st\qquad\text{for every }t\in\RR.
\]

\paragraph{Step 2: Strong duality and multiplier attainment.}
Every $Q\in\mQ$ is feasible for the perturbation problem at $t=g(Q)$.
Consequently the supporting inequality gives
\[
 \ell(Q,\pi,\bu)\geq h(g(Q))\geq h(0)+s g(Q).
\]
With $\lambda_0=-s\geq0$, this becomes
$\ell(Q,\pi,\bu)+\lambda_0g(Q)\geq h(0)$ for every $Q$.
Taking the minimum over $Q$ proves
$q_{\pi,\bu}(\lambda_0)\geq h(0)$.
For the reverse inequality, let $Q^*$ minimize the feasible primal
problem. For any $\lambda\geq0$, feasibility gives $g(Q^*)\leq0$,
so
\[
 q_{\pi,\bu}(\lambda)
  \leq\ell(Q^*,\pi,\bu)+\lambda g(Q^*)
  \leq\ell(Q^*,\pi,\bu)=h(0).
\]
This is weak duality. At $\lambda_0$ the two inequalities coincide,
and hence
\begin{equation}
  \label{eq: tube strong duality}
  D(\pi,\bu)=V(\pi,\bu)
       \quad\text{for every }(\pi,\bu,\bv)\in\mathcal T_r.
\end{equation}
The construction establishes strong duality relative to the convex
domain $\mQ$. It requires no interior of $\mQ$ in the ambient function
space.

We also give a quantitative bound for attaining multipliers.
Using the strictly feasible $Q^0=Q_{\pi,\bv}$ and
$|\ell(Q^0,\pi,\bu)|\leq N$ gives
\[
 q_{\pi,\bu}(\lambda)\leq N-\lambda\alpha/2,
 \qquad q_{\pi,\bu}(0)\geq-N.
\]
For $\lambda>4N/\alpha$, the first bound is strictly less than $-N$,
so such a multiplier cannot maximize $q_{\pi,\bu}$.
For completeness, $q_{\pi,\bu}$ is upper semicontinuous: if
$\lambda_n\to\lambda$, choose a minimizing function $Q_\lambda$
at $\lambda$ and use
\[
 \limsup_nq_{\pi,\bu}(\lambda_n)
 \leq\lim_n\{\ell(Q_\lambda,\pi,\bu)+\lambda_n g(Q_\lambda)\}
 =q_{\pi,\bu}(\lambda).
\]
A maximizing sequence on $[0,4N/\alpha]$ has a convergent subsequence,
and upper semicontinuity makes its limit a maximizer.
Thus the dual maximum is attained in $[0,4N/\alpha]$.
This interval may grow with the sample size. Its bound is used only
for attainment and is never multiplied by a concentration error.

The attaining multiplier may also be selected measurably. For the
countable dense subset $\{Q_j:j\geq1\}$ used above, continuity in $Q$
gives
\[
 q_{\pi,\bu}(\lambda)
 =\inf_{j\geq1}\{\ell(Q_j,\pi,\bu)
              +\lambda[\widehat c(Q_j,\pi,\bu)-\alpha]\}.
\]
Moreover, $|\widehat c(Q,\pi,\bu)|\leq2BC_{\mF}$ implies
\[
 |q_{\pi,\bu}(\lambda)-q_{\pi,\bu}(\lambda')|
 \leq(2BC_{\mF}+\alpha)|\lambda-\lambda'|.
\]
After the primal triple has been selected measurably, the countable
infimum is measurable in the data and this bound gives continuity in
$\lambda$. The measurable maximum theorem on the compact interval
$[0,4N/\alpha]$ therefore supplies a measurable attaining multiplier.

\paragraph{Step 3: Lifting the local primal maximizer.}
Choose the local maximizer $(\wh\pi,\wh\bu,\wh\bv)$ constructed in
Theorem \ref{thm: oracle feasibility} and an attaining multiplier
$\wh\lambda\in[0,4N/\alpha]$. That theorem supplies a relative open
neighborhood $\mathcal U$ of the triple contained in $\mathcal T_r$
on which $H(\pi,\bu,\bv)\leq H(\wh\pi,\wh\bu,\wh\bv)$.
For every $(\pi,\bu,\bv)\in\mathcal U$ and every $\lambda\geq0$,
\begin{align*}
 q_{\pi,\bu}(\lambda)-\pen_\mu(\bu,\bv)
 &\leq D(\pi,\bu)-\pen_\mu(\bu,\bv)\\
 &=H(\pi,\bu,\bv)\\
 &\leq H(\wh\pi,\wh\bu,\wh\bv)\\
 &=q_{\wh\pi,\wh\bu}(\wh\lambda)
                         -\pen_\mu(\wh\bu,\wh\bv).
\end{align*}
The first inequality is the definition of $D$, the first equality
uses \eqref{eq: tube strong duality}, and the last equality also uses
attainment at $\wh\lambda$. In particular, the inequality holds on
$\mathcal U\times\{\lambda\geq0:|\lambda-\wh\lambda|<1\}$,
a relative neighborhood of the quadruple. It is therefore a local
maximizer of \eqref{eq: dual penalized policy learning}.
Set $(\wh\pi^\dag,\wh\bu^\dag,\wh\bv^\dag,\wh\lambda^\dag)
=(\wh\pi,\wh\bu,\wh\bv,\wh\lambda)$.
Its policy and latent variables equal those of the constructed primal
solution. Their regret is consequently bounded by $C_1\alpha$ in
absolute value on $E_{NT}$ for all sufficiently large sample sizes.
Since $\Pr(E_{NT}^c)\to0$, this proves the stated existence probability
and $\bigOp(\alpha)$ regret bound.
\end{proof}

\section{Auxiliary Results and Proofs}

\blue{
\begin{lemma}
  \label{lem: linear MDP}
Under the linear-MDP conditions of Example \ref{ex: linear MDP}, including
bounded measurable features, finite signed component measures, and full feature span,
for every policy $\pi$ there is a unique vector $w_\pi\in\RR^d$ such that
\begin{equation*}
Q_\pi(s,a)=\bpsi(s,a)^\top w_\pi.
\end{equation*}
Writing
\begin{equation*}
M_\pi=\int_{\mS}\bmu(ds_+)\sum_{a\in\mA}
                    \pi(a\given s_+)\bpsi(s_+,a)^\top,
\end{equation*}
the matrix $I_d-\gamma M_\pi$ is invertible and
$w_\pi=(I_d-\gamma M_\pi)^{-1}\btheta$.
\end{lemma}

\begin{proof}
We suppress the individual index. Write
$V_\pi(s)=\sum_a\pi(a\mid s)Q_\pi(s,a)$ for the true state value.
Since the transition kernel is a probability kernel and rewards are
bounded, $\|V_\pi\|_\infty\leq B$ and $\|Q_\pi\|_\infty\leq B$.
For each signed component measure $\bmu_j$, finite total variation gives
\[
 \left|\int V_\pi(s_+)\bmu_j(ds_+)\right|
 \leq B\,|\bmu_j|(\mS)<\infty.
\]
The Bellman equation and the linear expressions for rewards and
transitions therefore imply
\begin{align*}
 Q_\pi(s,a)
 &=\bpsi(s,a)^\top\btheta
       +\gamma\int V_\pi(s_+)\bpsi(s,a)^\top\bmu(ds_+)\\
 &=\bpsi(s,a)^\top
       \left\{\btheta+\gamma\int V_\pi(s_+)\bmu(ds_+)\right\}.
\end{align*}
Define the vector in braces to be $w_\pi$. This proves existence of a
linear representation without assuming a matrix inverse in advance.
Averaging the representation over the action distribution gives
\[
 V_\pi(s_+)=\left\{\sum_a\pi(a\mid s_+)\bpsi(s_+,a)^\top\right\}w_\pi.
\]
All terms are integrable against the signed component measures because
$\bpsi$ is bounded. Substitution into the definition of $w_\pi$ yields
\begin{equation}
 \label{eq: linear coefficient fixed point}
 w_\pi=\btheta+\gamma M_\pi w_\pi.
\end{equation}

To prove invertibility, define
$\|x\|_\psi=\sup_{s,a}|\bpsi(s,a)^\top x|$ for $x\in\RR^d$.
Bounded features make this quantity finite. Homogeneity and the triangle
inequality follow from the same properties of the absolute value. If
$\|x\|_\psi=0$, then $x$ is orthogonal to every feature vector.
The full-span assumption gives $x=0$, so this is a norm.
More explicitly, choose $d$ linearly independent feature vectors and
form the invertible matrix $F$ with these vectors as rows. Then
\[
 \|x\|_2\leq\|F^{-1}\|_2\|Fx\|_2
       \leq\sqrt d\,\|F^{-1}\|_2\|x\|_\psi,
 \qquad \|x\|_\psi\leq\sup_{s,a}\|\bpsi(s,a)\|_2\|x\|_2.
\]
Thus $\|\bullet\|_\psi$ is equivalent to the Euclidean norm and makes
$\RR^d$ a complete normed space.

For any $x\in\RR^d$, the transition representation gives
\begin{align*}
 |\bpsi(s,a)^\top M_\pi x|
 &=\left|\int\PP(ds_+\mid s,a)\sum_{a_+}\pi(a_+\mid s_+)
                      \bpsi(s_+,a_+)^\top x\right|\\
 &\leq\int\PP(ds_+\mid s,a)\sum_{a_+}\pi(a_+\mid s_+)
                      |\bpsi(s_+,a_+)^\top x|\\
 &\leq\|x\|_\psi.
\end{align*}
Both the transition measure and the policy weights are nonnegative and
have total mass one. Taking the supremum over $(s,a)$ proves
$\|M_\pi x\|_\psi\leq\|x\|_\psi$. In the induced matrix norm this is
$\|M_\pi\|_\psi\leq1$ and hence
$\|M_\pi^t\|_\psi\leq1$ for every integer $t\geq0$.

Let $S_m=\sum_{t=0}^m\gamma^tM_\pi^t$. If $n>m$, then
\[
 \|S_n-S_m\|_\psi
 \leq\sum_{t=m+1}^n\gamma^t
 \leq\frac{\gamma^{m+1}}{1-\gamma}.
\]
Thus $S_m$ converges in matrix norm to a matrix $S$. Multiplication of the
finite geometric sum, on either side, gives
\[
 (I_d-\gamma M_\pi)S_m
   =S_m(I_d-\gamma M_\pi)
   =I_d-\gamma^{m+1}M_\pi^{m+1}.
\]
The norm of the last remainder is at most $\gamma^{m+1}$.
Passing to the limit proves
$(I_d-\gamma M_\pi)S=S(I_d-\gamma M_\pi)=I_d$.
Therefore $I_d-\gamma M_\pi$ is invertible, with inverse $S$.
Applying this inverse to \eqref{eq: linear coefficient fixed point}
proves the displayed formula for $w_\pi$.
Finally, if two vectors represent the same Q-function, their difference
has zero $\|\bullet\|_\psi$ norm and hence is zero. This proves uniqueness.
\end{proof}

}

\begin{assumption}
  \label{ass: general beta-mixing}
  For each $i\in[N]$, the process $\{X_t^i\}_{t\geq0}$ takes values in a
  standard Borel space $\mX$, is stationary, and has absolute regularity
  coefficients satisfying $\beta_i(j)\leq\beta_0\exp(-\zeta j)$ for all
  integers $j\geq1$, with common constants $\beta_0\geq0$ and $\zeta>0$.
  The $N$ trajectories are independent. Their marginal distributions may
  differ. We use the convention
  \[
    \beta_i(j)=\sup_{q\geq0}
      \beta\bigl(\sigma(X_0^i,\ldots,X_q^i),
                  \sigma(X_{q+j}^i,X_{q+j+1}^i,\ldots)\bigr).
  \]
  For a standard Borel random element $Y$ and a sigma field $\mathcal A$,
  the normalization of $\beta$ is
  \[
    \beta(\mathcal A,\sigma(Y))
      =\EE\left\|\mathcal L(Y\mid\mathcal A)-\mathcal L(Y)
            \right\|_{\mathrm{TV}},\qquad
    \|P-Q\|_{\mathrm{TV}}=\sup_E|P(E)-Q(E)|.
  \]
\end{assumption}

\begin{lemma}[Generalized Berbee's Coupling Lemma]
  \label{lem:berbee}
  Let $Y_1,\ldots,Y_m$ be random elements in standard Borel spaces. On an
  extension of the probability space there exist independent random elements
  $\widetilde Y_1,\ldots,\widetilde Y_m$, with
  $\widetilde Y_j\stackrel{d}{=}Y_j$, such that
  \[
    \Pr(\widetilde Y_j\ne Y_j)
    \leq \beta\bigl(\sigma(Y_1,\ldots,Y_{j-1}),\sigma(Y_j)\bigr),
    \qquad j=2,\ldots,m,
  \]
  and $\widetilde Y_1=Y_1$. Here $\beta(\mathcal A,\mathcal B)$ denotes
  the absolute regularity coefficient between the two sigma fields.
  Independent input sequences may be coupled using independent auxiliary
  randomness.
\end{lemma}

\begin{proof}
We first describe the conditional coupling used at each step. Let $P$ be
an observed past, let $K_p$ be a regular conditional law of a random
element $Y$ given $P=p$, and let $\nu=\mathcal L(Y)$. Such conditional
laws exist because the spaces are standard Borel. With respect to
$K_p+\nu$, choose jointly measurable densities $k_p$ and $v_p$ of $K_p$
and $\nu$. Define the common subprobability measure and its mass by
\[
  \kappa_p(dy)=\min\{k_p(y),v_p(y)\}(K_p+\nu)(dy),\qquad
  c_p=\kappa_p(\mY).
\]
Thus $1-c_p=\|K_p-\nu\|_{\mathrm{TV}}$. Introduce two independent
uniform random variables, independent of all the original variables.
Conditional on $(P,Y)=(p,y)$, retain $y$ with probability
\[
 h_p(y)=\frac{\min\{k_p(y),v_p(y)\}}{k_p(y)}
       \quad\text{when }k_p(y)>0.
\]
The value when $k_p(y)=0$ is immaterial and can be set to zero. Use the
first uniform variable to make this decision. Upon retention, set
$\widetilde Y=y$. Upon rejection, use the second uniform variable to
sample from $(\nu-\kappa_p)/(1-c_p)$. A measurable sampler exists for a
probability kernel on a standard Borel space. If $c_p=1$, rejection has
probability zero, so the residual sampler can be defined arbitrarily.
For every measurable set $E$, the conditional law of the output satisfies
\[
 \Pr(\widetilde Y\in E\mid P=p)
 =\kappa_p(E)+(1-c_p)
                  \frac{\nu(E)-\kappa_p(E)}{1-c_p}
 =\nu(E),
\]
where the case $c_p=1$ is read directly as $\kappa_p(E)=\nu(E)$.
Also, a mismatch requires rejection, so
\[
 \Pr(\widetilde Y\ne Y\mid P=p)
       \leq 1-c_p=\|K_p-\nu\|_{\mathrm{TV}}.
\]

To apply this construction successively, put $\widetilde Y_1=Y_1$ and
adjoin independent uniform pairs $U_2,\ldots,U_m$ that are jointly
independent of the entire original vector $(Y_1,\ldots,Y_m)$. At step
$j$, take $P=(Y_1,\ldots,Y_{j-1})$ and use $U_j$ in the construction
above. Write
\[
 \mathcal H_{j-1}
 =\sigma(Y_1,\ldots,Y_{j-1},U_2,\ldots,U_{j-1}).
\]
Independence of the raw auxiliary variables from the entire original
vector implies
\[
 \mathcal L(Y_j\mid\mathcal H_{j-1})
       =\mathcal L(Y_j\mid Y_1,\ldots,Y_{j-1}).
\]
The construction consequently gives
$\mathcal L(\widetilde Y_j\mid\mathcal H_{j-1})=\mathcal L(Y_j)$.
Each earlier copy is $\mathcal H_{j-1}$-measurable. Hence
$\widetilde Y_j$ is independent of all earlier copies. Induction proves
mutual independence and the required marginal laws. Integrating the
conditional mismatch bound gives the stated $\beta$ bound for every
$j$. Finally, copies for an input sequence are measurable functions of
that sequence and its auxiliary uniforms. Independent sequences with
independent auxiliary uniforms therefore produce independent collections
of copies.
\end{proof}

\begin{lemma}
  \label{lem:hoeffding-mixing}
  Suppose Assumption \ref{ass: general beta-mixing} holds, $N,T\geq1$,
  and $\mF$ is a class of measurable functions $f:\mX\to[-B,B]$, where
  $B>0$. Suppose that, for some $A\geq2$ and $D\geq1$,
  \[
    N(\epsilon B,\mF,\|\bullet\|_\infty)
    \leq (A/\epsilon)^D,\qquad 0<\epsilon\leq1.
  \]
  For $0<\delta<1$, define
  \[
    a=a(N,T,\delta)
    :=\min\left\{T,\max\left\{1,
       \left\lceil\zeta^{-1}
       \log\frac{8(\beta_0\vee1)NT}{\delta}\right\rceil
       \right\}\right\}.
  \]
  There is a universal constant $C$ such that, with probability at least
  $1-\delta$,
  \begin{equation}
    \label{eq: uniform mixing bound}
    \sup_{f\in\mF}\left|
      \frac1{NT}\sum_{i=1}^N\sum_{t=0}^{T-1} f(X_t^i)
      -\frac1N\sum_{i=1}^N\EE f(X_0^i)\right|
    \leq CB\sqrt{\frac{a}{NT}
                    \left\{D\log(2A)+\log\frac8\delta\right\}}.
  \end{equation}
  In particular, if $NT\geq2$ and $(NT)^{-2}\leq\delta<1$, this is at most
  \[
    C_{A,\beta_0}B
    \sqrt{\frac{(1+\zeta^{-1})D}{NT}
                 \log(eNT)\log\frac8\delta}.
  \]
  The first bound applies to every trajectory length $T$, including lengths
  shorter than the uncapped logarithmic block length. No lower bound on $T$
  is needed.
\end{lemma}

\begin{proof}
Write $n=NT$ and $q=\lceil T/a\rceil$. The blocks in each trajectory are
indexed by $j\in\{0,\ldots,q-1\}$, with time-index sets and lengths
\[
 I_j=\{ja,ja+1,\ldots,\min\{(j+1)a,T\}-1\},\qquad
 \ell_j=|I_j|=\min\{a,T-ja\}.
\]
In particular, the final block has length $T-(q-1)a\in\{1,\ldots,a\}$,
so no observations are discarded. For $r\in\{0,1\}$, set
$\mathcal J_r=\{j\in\{0,\ldots,q-1\}:j\equiv r\pmod 2\}$.
For each trajectory $i$, apply Lemma \ref{lem:berbee} to the ordered
sequence of block vectors $(X_t^i)_{t\in I_j}$ with $j\in\mathcal J_r$.
Denote the resulting copies by $(\widetilde X_t^{i,r})_{t\in I_j}$.
For a block $j$ with a predecessor of the same parity, all previous
blocks of that parity end by time $(j-1)a-1$, whereas block $j$ begins
at time $ja$. Monotonicity of $\beta$ under restriction of sigma fields
therefore bounds its coupling error by
\[
 \beta_i(a+1)\leq\beta_0e^{-\zeta(a+1)}
                      \leq\beta_0e^{-\zeta a}.
\]
The first block of each parity is copied exactly. Use independent
auxiliary randomness across trajectories. For each fixed parity, all
copied blocks are then mutually independent, including across
trajectories. Independence between the two parities is not required.

Let $\mathcal E$ be the event that every copied block agrees with its
original block. If $a<T$, the definition of $a$ gives
$a\geq\zeta^{-1}\log\{8(\beta_0\vee1)n/\delta\}$. There are at most
$Nq\leq NT=n$ blocks in both parities combined, so the union bound gives
\[
 \Pr(\mathcal E^c)
 \leq Nq\beta_0e^{-\zeta a}
 \leq n\beta_0\frac{\delta}{8(\beta_0\vee1)n}
 \leq\delta/8.
\]
If $a=T$, there is only one block per trajectory. These original blocks
are independent, their copies can be taken equal to them, and
$\Pr(\mathcal E^c)=0$.

Fix a parity $r$ and define
\[
 \mathcal B_r=\{(i,j):i\in[N],\ j\in\mathcal J_r\},\qquad
 S_r(f)=\frac1n\sum_{(i,j)\in\mathcal B_r}
 \left\{\sum_{t\in I_j}f(\widetilde X_t^{i,r})
                      -\EE\sum_{t\in I_j}f(X_t^i)\right\}.
\]
An empty parity contributes zero throughout. The block lengths satisfy
\[
 \sum_{(i,j)\in\mathcal B_r}\ell_j^2
       \leq a\sum_{(i,j)\in\mathcal B_r}\ell_j\leq an.
\]
For $d=\|f-g\|_\infty$, the uncentered block sum of $f-g$ lies in
$[-\ell_j d,\ell_j d]$. Hoeffding's exponential-moment bound for a
centered random variable whose range has length $L$ is
$\EE e^{\lambda Z}\leq e^{\lambda^2L^2/8}$. One direct verification is
to set $\psi(\lambda)=\log\EE e^{\lambda Z}$. Its second derivative is
the variance of $Z$ under the probability measure with density
$e^{\lambda Z}/\EE e^{\lambda Z}$. That measure is supported on the
same interval of length $L$, so its variance is at most $L^2/4$. Indeed,
the mean minimizes squared distance over constants, and squared distance
from the midpoint of this interval is at most $L^2/4$. Since
$\psi(0)=\psi'(0)=0$, integration of the second-derivative bound gives
$\psi(\lambda)\leq\lambda^2L^2/8$ for either sign of $\lambda$.
Applying this bound to the independent blocks gives, for every real
$\lambda$,
\[
 \EE e^{\lambda\{S_r(f)-S_r(g)\}}
 \leq\exp\left\{\frac{\lambda^2d^2}{2n^2}
                     \sum_{(i,j)\in\mathcal B_r}\ell_j^2\right\}
 \leq\exp\left\{\frac{\lambda^2ad^2}{2n}\right\}.
\]
Markov's inequality with $\lambda=nx/(ad^2)$ and then the same argument
for the negative increment prove, for $d>0$,
\begin{equation}
 \label{eq: mixing chaining increment}
 \Pr\{|S_r(f)-S_r(g)|>x\}
 \leq2\exp\left\{-\frac{nx^2}{2ad^2}\right\}.
\end{equation}
When $d=0$, the increment is zero. The same exponential-moment and tail
bounds hold for $S_r(f)$ itself with $d$ replaced by $B$.

We now give the finite-net argument, including the bound on each maximum.
Suppose $Z_1,\ldots,Z_M$ satisfy
$\EE e^{\lambda Z_k}\leq e^{\lambda^2\sigma^2/2}$ for all real $\lambda$.
For $\lambda>0$, without any independence requirement, we have
\[
 \exp\{\lambda\EE\max_{k\leq M}|Z_k|\}
 \leq\EE e^{\lambda\max_{k\leq M}|Z_k|}
 \leq\sum_{k=1}^M\EE(e^{\lambda Z_k}+e^{-\lambda Z_k})
 \leq2M e^{\lambda^2\sigma^2/2}.
\]
Taking logarithms and choosing
$\lambda=\sqrt{2\log(2M)}/\sigma$ gives
\[
 \EE\max_{k\leq M}|Z_k|\leq\sigma\sqrt{2\log(2M)}.
\]
The case $\sigma=0$ follows directly because all the variables then
vanish almost surely.

Take finite $B2^{-j}$-nets $\mathcal F_j\subseteq\mF$, for $j\geq0$,
with $|\mathcal F_j|\leq(A2^j)^D$. Covering numbers here use centers in
the class. For each $f\in\mF$, select $f_j\in\mathcal F_j$ with
$\|f-f_j\|_\infty\leq B2^{-j}$. Then
\[
 \|f_j-f_{j-1}\|_\infty\leq3B2^{-j},\qquad j\geq1.
\]
At level $j\geq1$, only pairs in the finite set
\[
 \mathcal P_j=\{(u,v)\in\mathcal F_j\times\mathcal F_{j-1}:
                       \|u-v\|_\infty\leq3B2^{-j}\}
\]
are needed. Let $M_j=|\mathcal P_j|$ and set $L=\log(2A)$. The covering
bounds and $D\geq1$ imply
\[
 \log(2|\mathcal F_0|)\leq DL,\qquad
 \log(2M_j)\leq2D(L+j)\quad\text{when }M_j\geq1.
\]
The maximum over an empty set of pairs is taken to be zero. The
exponential-moment bounds above give
\begin{align*}
 \EE\max_{u\in\mathcal F_0}|S_r(u)|
     &\leq B\sqrt{\frac{2aDL}{n}},\\
 \EE\max_{(u,v)\in\mathcal P_j}|S_r(u)-S_r(v)|
     &\leq6B2^{-j}\sqrt{\frac{aD(L+j)}{n}}.
\end{align*}
To justify telescoping without a measurability assumption on an
uncountable index set, observe the deterministic continuity bound
\[
 |S_r(f)-S_r(g)|
 \leq\frac{2}{n}\sum_{(i,j)\in\mathcal B_r}\ell_j
                         \|f-g\|_\infty
 \leq2\|f-g\|_\infty.
\]
Thus $S_r(f_j)\to S_r(f)$ uniformly in $f$, and
\[
 Z_r:=\sup_{f\in\mF}|S_r(f)|
      =\sup_{f\in\bigcup_{j\geq0}\mathcal F_j}|S_r(f)|
\]
is a measurable random variable, as a supremum of a countable family.
For each positive integer $J$, telescoping gives
\[
 Z_r\leq\max_{u\in\mathcal F_0}|S_r(u)|
   +\sum_{j=1}^J\max_{(u,v)\in\mathcal P_j}|S_r(u)-S_r(v)|
   +2B2^{-J}.
\]
Letting $J\to\infty$ and using monotone convergence for the nonnegative
series on the right proves
\begin{align*}
 \EE Z_r
 &\leq B\sqrt{\frac{aD}{n}}
       \left\{\sqrt{2L}+6\sum_{j=1}^\infty
                              2^{-j}\sqrt{L+j}\right\}\\
 &\leq20B\sqrt{\frac{aDL}{n}}.
\end{align*}
For the last inequality, use $L\geq\log4>1$,
$\sqrt{L+j}\leq\sqrt L+j$,
$\sum_{j\geq1}2^{-j}=1$, and $\sum_{j\geq1}j2^{-j}=2$.

It remains to convert this expectation bound into a probability bound.
Replacing a copied block of length $\ell_j$ changes $S_r(f)$ by at most
$2B\ell_j/n$ for every $f$. The same bound holds for $Z_r$, because
$|\sup_f|x_f|-\sup_f|y_f||\leq\sup_f|x_f-y_f|$. Thus the squared
bounded-difference constants sum to at most
\[
 \sum_{(i,j)\in\mathcal B_r}(2B\ell_j/n)^2
       \leq4B^2a/n.
\]
For completeness, revealing the independent blocks one at a time gives
a Doob martingale for $Z_r$. Each centered martingale increment has
conditional range of length at most the corresponding constant
$2B\ell_j/n$. Conditional application of Hoeffding's exponential-moment
bound and iteration give
\[
 \EE e^{\lambda(Z_r-\EE Z_r)}
       \leq e^{\lambda^2B^2a/(2n)}.
\]
Markov's inequality, optimized over $\lambda>0$, consequently yields
\[
 \Pr\{Z_r-\EE Z_r>t\}
       \leq\exp\{-nt^2/(2B^2a)\}.
\]
Set $t=B\sqrt{2a\log(4/\delta)/n}$. With probability at least
$1-\delta/4$, the bound for this parity is therefore
\[
 Z_r\leq20B\sqrt{\frac{aDL}{n}}
                +B\sqrt{\frac{2a}{n}\log\frac4\delta}.
\]

On $\mathcal E$, summing the two centered parity processes gives the
original centered empirical process. Stationarity supplies the equality
\[
 \frac1n\sum_{i=1}^N\sum_{t=0}^{T-1}\EE f(X_t^i)
       =\frac1N\sum_{i=1}^N\EE f(X_0^i).
\]
The two parity bounds and the coupling event hold simultaneously with
probability at least
$1-\delta/4-\delta/4-\delta/8=1-5\delta/8\geq1-\delta$.
On that event, the desired supremum is at most
\[
 Z_0+Z_1
 \leq40B\sqrt{\frac{aDL}{n}}
           +2B\sqrt{\frac{2a}{n}\log\frac4\delta}
 \leq44B\sqrt{\frac an
                     \left\{DL+\log\frac8\delta\right\}}.
\]
This proves \eqref{eq: uniform mixing bound} with, for example, $C=44$.

Finally, suppose $n\geq2$ and $\delta\geq n^{-2}$. Then
\[
 \log\frac{8(\beta_0\vee1)n}{\delta}
       \leq\log\{8(\beta_0\vee1)\}+3\log n.
\]
The ceiling and truncation defining $a$ imply, with
$c_{\beta_0}=4+\log\{8(\beta_0\vee1)\}$,
\[
 a\leq1+\zeta^{-1}
                   \log\frac{8(\beta_0\vee1)n}{\delta}
   \leq c_{\beta_0}(1+\zeta^{-1})\log(en).
\]
Also $\log(8/\delta)\geq\log8$ and $D\geq1$, so
\[
 D\log(2A)+\log(8/\delta)
 \leq\left\{1+\frac{\log(2A)}{\log8}\right\}
                  D\log(8/\delta).
\]
Substitution into \eqref{eq: uniform mixing bound} proves the last claim
with
$C_{A,\beta_0}=44\sqrt{c_{\beta_0}\{1+\log(2A)/\log8\}}$.
\end{proof}

\section{More Details in Numerical Studies}
\blue{
In the empirical study, we use neural networks to approximate the function classes $\mQ$, ratio class $\mF$ and policy class $\Pi$. The structure of the neural network is shown in Figure \ref{fig:nn_structure}. To represent heterogeneity, an intercept-free linear layer $U = [u^1 | \dots | u^i| \dots | u^N]$ encodes subject $i$'s one-hot encoder $e_i = [0,\dots, 1,\dots,0]^{\top}$ to latent vector $u^i$, so that $Ue_i = u^i$. Then the Feature Encoder (a 1-layer net) takes the concatenated inputs of $u^i$ and state $s$ as the full state, and outputs an encoded feature vector. Here, $\mF$-net and $\mQ$-net are two separate neural networks that take the action and encoded feature vector as input and output the visiting probability ratio and Q-function, respectively. The policy network $\Pi$-net takes the encoded feature vector as input and outputs the probability mass function for all actions $a\in\mA$ with a top layer of softmax activation.

\begin{figure}[H]
  \centering
  \includegraphics[width=0.8\textwidth]{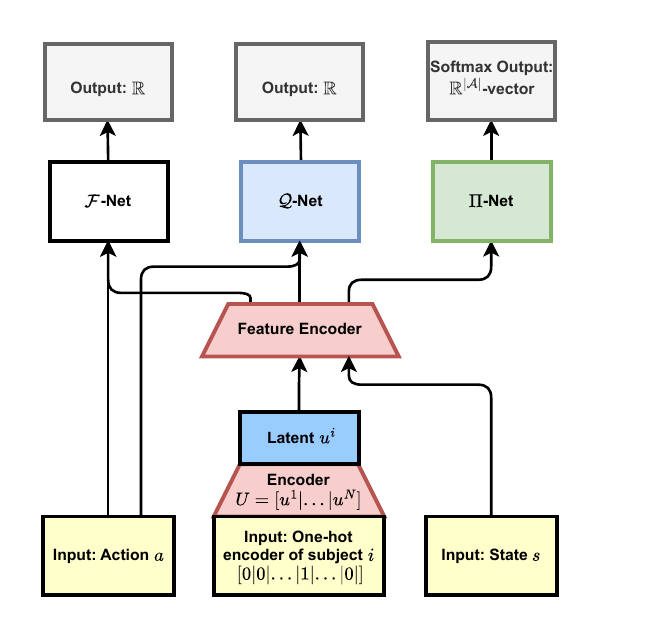}
  \caption{The structure of the neural network in the empirical study.}
  \label{fig:nn_structure}
\end{figure}

All $\mF$-net, $\mQ$-net and $\Pi$-net are fully-connected feed forward neural
networks of 2 layers with ReLU activation functions and one residual connection
from input to the layer before output (ResNet). For either architecture, polynomial
uniform entropy can be obtained on a bounded input domain by restricting
all parameters to a compact set. Specifically, if a network has $w$
parameters and is uniformly Lipschitz in its parameter vector on that set,
a parameter-space grid gives
$N(\epsilon,\mH,\|\bullet\|_\infty)\leq(C/\epsilon)^w$.
Fixed finite architectures with Lipschitz activations have this property
when inputs and weights are bounded. Output bounds and latent Lipschitz
bounds must also be enforced to satisfy Assumption \ref{ass: spaces}.
These observations do not make a neural-network Q-function class convex,
and the unconstrained numerical implementation is not automatically covered
by the local convex-class guarantees.

}

\section{Additional Numerical Results}
More cases for simulation setting in Section \ref{sec: sim1}.
\begin{figure}[H]
  \centering
  \includegraphics[width=0.9\textwidth]{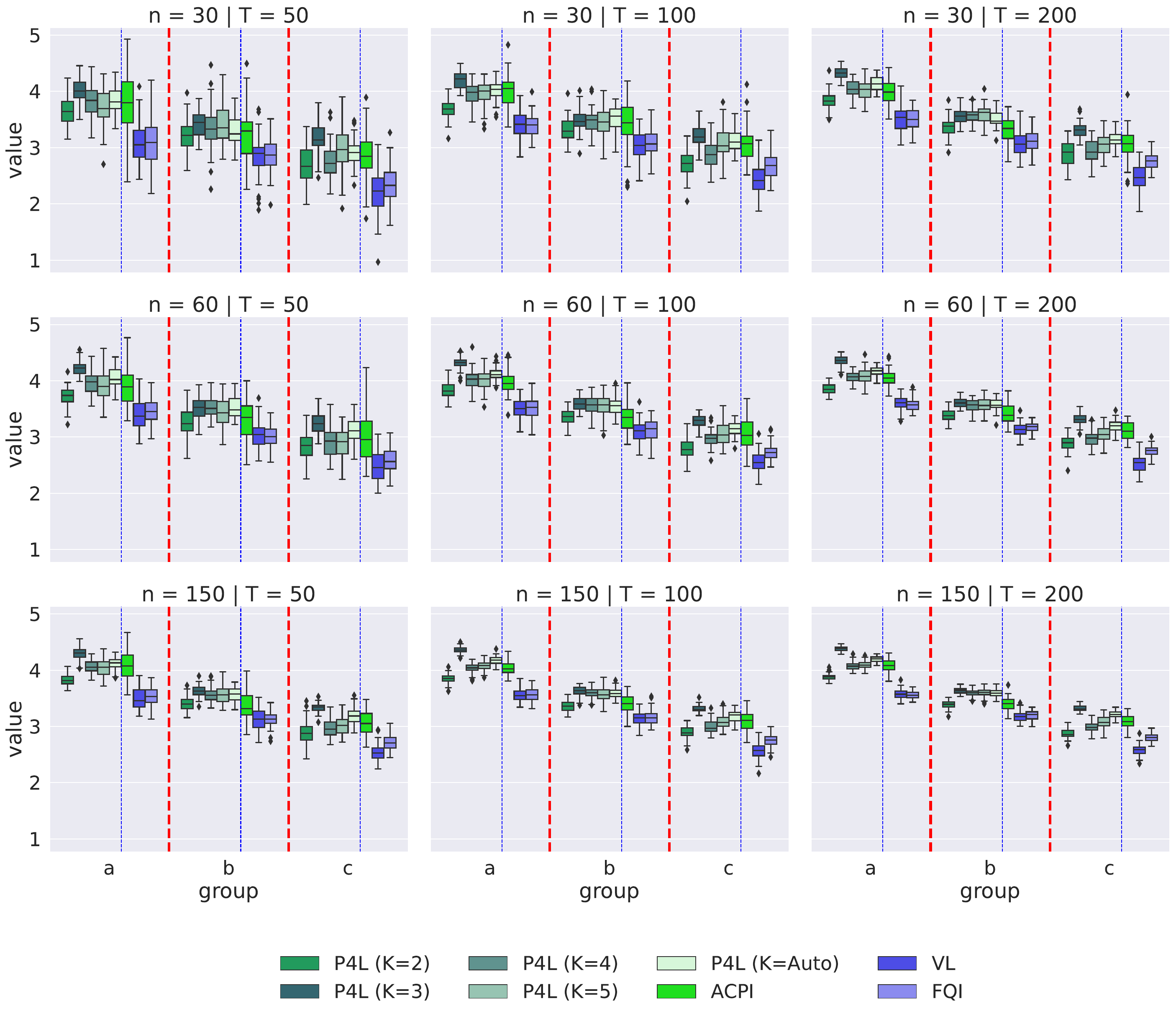}
  \caption{Boxplots for values of estimated policies for $n=30,60$ and $150$ and
    $T=50,100$ and 200. Red dashed lines separate
    three groups (a), (b) and (c). In each group, the blue dotted line separates
    the P4L method of different $K$ from the benchmark methods.}
  \label{fig:simu_values_full}
\end{figure}

\begin{figure}[H]
  \centering
  \includegraphics[width=0.3\textwidth]{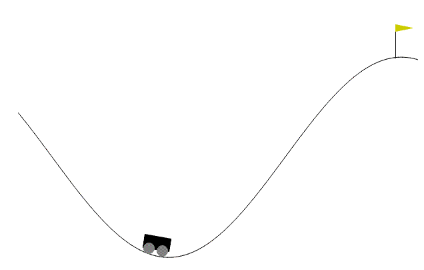}
  \caption{MountainCar and CartPole environments.}
  \label{fig:gym_envs2}
\end{figure}

\noindent \textbf{MountainCar.} In MountainCar, the goal is to drive a
under-powered car to the top of a hill by taking the least number of steps (terminate at 500 steps).
\begin{itemize}
    \item \emph{Observation.} We observe $x_t$, $\Dot{x}_t$: the position and velocity of the car, respectively.
    \item \emph{Actions.} There are three possible actions $\{0,1,2\}$:   (-1) accelerate to the left, (0) do nothing, and (1) accelerate to the right.
    \item {\em Reward}. The reward is defined as
    $
    R_t = \II(x_t\geq 0.5) - \II(x_t<0.5).
    $
    \item {\em Behavior policy.} $\text{sign}(\Dot{x}_t)$ with probability 0.8
        and 0 with probability 0.2.
    \item {\em Environment heterogeneity.} We vary the gravity within the range
        $[0.01,0.035]$.  With a weaker gravity, the environment is trivially
        solved by directly moving to the right. While with a stronger gravity,
        the car must drive left and right to build up enough momentum.
\end{itemize}

\begin{table}[htbp]
\centering
\begin{tabular}{c|ccc}
  \toprule
& \multicolumn{3}{c}{MountainCar (gravity)} \\
Settings & (0.01) & (0.025) & (0.035) \\
\midrule
P4L $(K=2)$ & -106.3±44.1 & -432.3±117 & -471.8±43.0 \\
P4L $(K=3)$ & -56.80±1.83 & \textbf{-189.4}±6.44 & \textbf{-210.6}±4.27\\
P4L $(K=4)$ & -74.23±16.5 & -492.3±13.3 & -500.0±0.0 \\
P4L $(K=5)$ & -67.60±22.3 & -267.8±202 & -295.1±180 \\
P4L $(K=Auto)$ & -63.20±4.3 & -233.1±24.1 & -263.1±28.0 \\
\midrule
ACPI & \textbf{-44.79}±0.08 & -380.7±170 & -500.0±0.0 \\
FQI & -176.1±45.2 & -316.4±26.4 & -362.9±17.9 \\
VL & -373.0±33.5 & -500.0±0.0 & -500.0±0.0 \\
\bottomrule
\end{tabular}
\caption{Values of learned policies from four RL methods on MountainCar with
  different settings.}
\label{tab:performance_mountaincar}
\end{table}

\end{document}